\documentclass{article}

\usepackage[final]{corl_2026} 

\usepackage{graphicx} 
\usepackage{subcaption}
\usepackage{nicematrix}
\usepackage{wrapfig}
\usepackage{caption}
\usepackage{algorithm}
\usepackage{algorithmic}
\usepackage{float}
\usepackage{amssymb}
\usepackage{amsmath}
\usepackage{dsfont}
\usepackage{amsthm}
\usepackage{bm}
\usepackage{mathtools}
\usepackage{xcolor}
\usepackage{xspace}
\usepackage[toc]{appendix}
\newtheorem{assumption}{Assumption}
\newtheorem{definition}{Definition}
\newtheorem{lemma}{Lemma}
\newtheorem{theorem}{Theorem}

\newtheorem{example}{Example}

\usepackage{paralist}
\usepackage{bbm}
\usepackage{booktabs}

\usepackage{tabularx}
\usepackage{array}

\usepackage{booktabs}
\usepackage{float}
\usepackage{array}
\usepackage[table]{xcolor}
\usepackage{caption}
\usepackage{adjustbox}

\definecolor{TableRuleGrey}{gray}{0.82}
\definecolor{AdaptiveBlue}{RGB}{244,248,255}

\newcommand{\ElegantTableSetup}{%
  \arrayrulecolor{TableRuleGrey}%
  \setlength{\heavyrulewidth}{0.55pt}%
  \setlength{\lightrulewidth}{0.35pt}%
  \setlength{\cmidrulewidth}{0.35pt}%
  \renewcommand{\arraystretch}{1.14}%
}

\title{When Should a Robot Replan? Regret-Guided Update Scheduling in Time-Varying MDPs}

\author{ Negin Musavi, Gokul Puthumanaillam, \\[0.5em] \textbf{Ruben Hernandez, William Schafer, and Melkior Ornik}\\[1.0em] {\small Department of Aerospace Engineering,}\\ {\small University of Illinois Urbana--Champaign, Champaign, IL 61820, USA} \\ {\tt\small nmusavi2@illinois.edu, gokulp2@illinois.edu,} \\ {\tt\small rubenjh2@illinois.edu, wes6@illinois.edu, mornik@illinois.edu} }

\begin{document}
\maketitle


\begin{abstract}
Robots operating in non-stationary environments must continually adapt their policies as the dynamics drift, but onboard energy and compute budgets cap how often a full state estimation and re-planning step can
be performed. This raises a question: \emph{when}, along a horizon, should a robot spend its limited budget? We formulate this problem in time-varying Markov decision processes (TVMDPs) with a known bound on the rate of transition drift. We model execution as a \emph{skip-update} scheme in which, at chosen update times, the agent estimates the transition kernel by maximum likelihood and computes a finite-horizon policy, and between updates reuses this policy under a propagated state estimate. We analyze the dynamic regret of this scheme and show how it
grows during skip intervals in terms of the properties of the TVMDP and the skip lengths; the resulting bound answers the opening question via an online, regret-guided update rule that allocates the budget adaptively. We evaluate the rule in a simulated Mars-rover navigation task with time-varying slip dynamics and on a Crazyflie quadrotor in indoor obstacle fields. Adaptive allocation outperforms other budgeted baselines.
\vspace{-0.5em}
\end{abstract}



\section{Introduction}\label{sec:intro}
\vspace{-0.75em}
Autonomous robots are increasingly deployed in non-stationary
environments. Examples arise across robotics: marine vehicles navigating through time-varying ocean currents~\citep{liu2018solution,
duckworth2021time, zhang2023realtime}, aerial vehicles operating in
changing wind fields~\citep{liu2018solution}, and mobile service robots planning navigation actions through the ebb and flow of human activity~\citep{pulidofentanes2015nowlater}. In such settings, the robot's policy needs to adapt as the dynamics evolve. At the same time, robots operate under limited onboard resources---energy, sensing, and computation---that restrict how often a full state estimation and re-planning step can be carried out. More frequent updates improve adaptation to current conditions but consume more of the available budget, while less frequent updates conserve resources at the cost of
operating on an outdated model. A natural question in this setting is:
\vspace{-0.5em}
\begin{center}
\emph{How should the limited update budget be allocated over time when the dynamics change?}
\end{center}
\vspace{-0.5em}

TVMDPs provide a natural framework for sequential decision-making under evolving dynamics. Existing work broadly divides into two settings depending on whether the transition model is available to the agent or must be learned from interaction. When the transition model is known or can be forecast, the focus is on planning rather than estimation. \citet{liu2018solution} introduce the TVMDP framework and develop a spatial-temporal value propagation method for robots navigating under known time-varying dynamics. \citet{duckworth2021time} extend this to semi-MDPs with GP-learned dynamics and time-bounded mission specifications. \citet{zhang2024predictive} study a model predictive dynamic 
programming approach with look-ahead forecasts, showing regret decreases exponentially with the forecast horizon under a finite-time mixing condition.

When the transition kernels are unknown, two estimation strategies have been pursued. One uses forgetting mechanisms --- sliding windows, periodic restarts, 
or confidence widening --- to track drift while planning optimistically: \citet{cheung2020reinforcement} introduce confidence widening under bounded 
variation budgets, \citet{chen2022near} refine this for stochastic shortest paths with separate restart windows for costs and transitions, \citet{wei2021non} remove the need for prior knowledge of the variation budget via a black-box 
meta-algorithm, and \citet{damoptimal} shows that low-rank structure in the drift improves regret rates. A complementary strategy imposes the drift bound directly as a constraint in the maximum likelihood estimation step without 
discarding historical data; \citet{ornik2021learning} develop this for TVMDPs with unknown transitions, and \citet{puthumanaillam2024weathering} extend it 
to partially observable settings. A separate line of work studies reward nonstationarity with fixed or revealed transition dynamics: \citet{fei2020dynamic} study policy optimization under adversarially changing rewards with fixed 
unknown transitions, while \citet{li2019online, li2019online2} consider the setting where both transitions and rewards are revealed after each step and use online value iterations to adapt to the changing environment. Across all of these lines, a common implicit assumption is that the agent observes its state and recomputes its policy at every time step. The question 
of how a fixed budget of updates should be allocated over time has not been addressed.

In this paper, we consider TVMDPs with unknown transition kernels but known drift bounds, and study how an agent with a limited budget---one that caps the total number of state observations and replanning steps over the horizon---should allocate that budget. We model execution as a \emph{skip-update} scheme in which, at chosen update times, the agent estimates the transition kernel by maximum likelihood and computes a finite-horizon policy, and between updates reuses this policy under a propagated state estimate. We analyze the dynamic regret of this scheme and show how it grows during skip intervals as a function of local drift and skip length; the resulting
bound motivates an online, regret-guided update rule that allocates the budget adaptively. The closest related work is~\citep{lee2024pausing}, which also studies reinforcement learning with intermittent policy
updates in non-stationary environments; our work differs in that it
exploits bounded temporal variation in the estimation step, couples
estimation with planning, and derives the update schedule from a regret bound.
\vspace{-0.8em}
\paragraph{Contributions.} Our contributions are threefold: 
(i)~a formulation of the budgeted TVMDP problem and a skip-update 
algorithm that estimates dynamics and plans only at a subset of time 
steps, executing under a propagated state estimate between updates; 
(ii)~a dynamic regret bound that decomposes into update-time and 
skip-interval contributions, with the skip-interval 
analysis---characterizing how regret accumulates under a stale policy 
and propagated state estimate---as the first technical contribution; 
and (iii)~an online, regret-guided update rule derived directly from 
this bound that allocates the available budget adaptively, as the 
second technical contribution. We evaluate the rule in simulated 
Mars-rover navigation with time-varying slip dynamics and on a 
Crazyflie quadrotor in an indoor obstacle field, showing adaptive allocation outperforms budgeted baselines.
\vspace{-0.8em}
\paragraph{Outline.} Section~\ref{sec:problem} formalizes the
TVMDP setting and the skip-update algorithm.
Section~\ref{sec:main_results} presents the dynamic regret
analysis and the adaptive update rule.
Section~\ref{sec:experiments} reports simulation and hardware
experiments. Section~\ref{sec:conclusion} provides concluding
remarks and discusses limitations.

\section{Problem Formulation}\label{sec:problem}

We consider a finite-horizon TVMDP $\mathcal{M} = \big(\mathcal{S}, \mathcal{A}, T, \{P_t\}_{t=0}^{T-1}, \{r_t\}_{t=0}^{T-1}\big)$, where $\mathcal{S}$ and $\mathcal{A}$ denote finite state and action spaces, $T \in \mathbb{N}$ is the horizon length, $P_t : \mathcal{S}\times\mathcal{A}\times\mathcal{S} \to [0,1]$ is the transition kernel at time $t$, satisfying $\sum_{s' \in \mathcal{S}} P_t(s' \mid s,a) = 1$ for all $(s,a)\in\mathcal{S}\times\mathcal{A}$, and $r_t : \mathcal{S}\times\mathcal{A} \to \mathbb{R}$ is the reward function at time $t$. The transition kernels satisfy the bounded drift condition $|P_{t+1}(s' \mid s,a) - P_t(s' \mid s,a)| \le \varepsilon_t$ for all $s,s'\in\mathcal{S}$, $a\in\mathcal{A}$, and $t\ge0$, where $\{\varepsilon_t\}_{t\ge 0}$ bounds the rate of temporal variation of the dynamics. At each time step $t$, the agent selects an action $a_t \in \mathcal{A}$, transitions to $s_{t+1} \sim P_t(\cdot \mid s_t, a_t)$, and receives reward $r_t(s_t, a_t)$. We assume that the transition kernels $\{P_t\}_{t\geq 0}$ are unknown to the agent, whereas the variation bounds $\{\varepsilon_t\}_{t\geq 0}$ and the reward functions $\{r_t\}_{t\geq 0}$ are known.

\begin{example}[Mars-rover navigation]\label{ex:mars_rover}
A Mars-rover wants to reach a goal cell from a designated start on a discretized terrain grid. The state $s_t \in \mathcal{S}$ encodes the rover's grid cell, and the action $a_t \in \mathcal{A}$ is a commanded move (right, left, up, down, or stay). Motion is stochastic: the rover's wheels can slip on loose soil or sloped terrain, so a commanded action may land the rover in an unintended adjacent cell. The transition kernel $P_t(\cdot \mid s, a)$ assigns high probability to the commanded outcome and distributes the remaining mass across slip outcomes, with slip rates shaped by local terrain conditions such as soil composition and illumination~\cite{helmick2009terrain}. Because these conditions vary slowly along the traverse, $P_t$ drifts gradually with time at a rate bounded by $\{\varepsilon_t\}_{t \geq 0}$. The reward function $r_t$ penalizes states farther from the goal. In addition, the rover operates under a strict energy budget that limits both onboard computation and information acquisition, making it a natural instance of the budgeted update setting formalized below.
\end{example}

The agent is allowed to observe the realized transition $(s_{t-1}, a_{t-1}, s_t)$ at selected time steps $t \in \mathcal{T}_{\mathrm{upd}} \subset \{1,\dots,T-1\}$, and update its model and policy only at those times, seeking policies that maximize the expected cumulative reward subject to a cap on the total number of such updates. Concretely, the agent seeks a sequence of deterministic policies $\{\pi_t\}_{t=0}^{T-1}$\footnote{We restrict attention to deterministic policies; this is without loss of generality since finite-horizon MDPs admit a deterministic optimal policy~\citep{Puterman1994}.}, 
$\pi_t:\mathcal{S}\to\mathcal{A}$, that solve $\max_{\{\pi_t\}_{t=0}^{T-1}}\;
    \mathbb{E}\!\left[\sum_{t=0}^{T-1} r_t\big(s_t, \pi_t(s_t)\big)\,\Big|\, 
    s_0 = s\right]$, such that $|\mathcal{T}_{\mathrm{upd}}| \le B$, where $B\in\mathbb{N}$ is the 
available resource budget. The central question is how these $B$ updates should be allocated over the horizon. To answer it, we first develop a skip-update algorithm allowing the agent to operate under an arbitrary update schedule, and introduce a performance metric to evaluate its performance. The remainder of this section develops both ingredients.

\subsection{Skip-update Algorithm}\label{sec:skip-update}

We first describe execution for a given set of update times; the online selection of these times is introduced in Section~3.2. Each update time $\tau_k$ denotes the time at which the transition tuple $(s_{\tau_k-1},a_{\tau_k-1},s_{\tau_k})$ is available and
model estimation and replanning are performed.
The budget $B$ counts these update events, with each event comprising transition acquisition, model estimation, and replanning. The skip-update scheme alternates between update phases and skip phases. It builds on the maximum-likelihood approach to learning and planning in TVMDPs developed
in~\citep{ornik2021learning,puthumanaillam2024weathering}, extending it to a budgeted setting in which model estimation and policy recomputation occur only at selected update times. The transition $(s_0,a_0,s_1)$ forms a common initialization step,
and the first update is completed at $\tau_0=1$.
Let $\mathcal{T}_{\mathrm{upd}}
=\{\tau_0<\dots<\tau_{N_T}\}\subseteq\{1,\dots,T-1\}$
denote the completed update times, including this initial update, and let $\mathcal{T}_{\mathrm{skip}}
=\{1,\dots,T-1\}\setminus\mathcal{T}_{\mathrm{upd}}$
denote the subsequent skip times.
For $t\in\{1,\dots,T-1\}$, define
$k(t)=\max\{k:\tau_k\le t\}$ as the index of the most recent completed update. For describing the final execution interval, set $\tau_{N_T+1}:=T$; this terminal marker is not an update time. We describe the update and skip phases below (see Appendix~\ref{app:skip_update} for details):

\begin{itemize}

\item \textbf{Update phase.}
At each $\tau_k$, the agent uses the collected transition
$(s_{\tau_k-1},a_{\tau_k-1},s_{\tau_k})$,
augments the dataset $\mathcal{D}_{\tau_k}
    =
    \{(s_{\tau_j-1},a_{\tau_j-1},s_{\tau_j})
    : 0\le j\le k\},$ and estimates the transition kernels by maximum likelihood
subject to the drift bound, resulting in the estimate
$\hat{P}_{\tau_k-1}$.

The agent solves a finite-horizon planning problem by backward induction, obtaining a policy sequence
$\{\pi_{k,h}^{\mathrm{plan}}\}_{h=0}^{H_k-1}$
that is optimal for every starting state $s\in\mathcal{S}$:
\begin{equation}\label{eq:policy-opt-finite}
    \{\pi_{k,h}^{\mathrm{plan}}\}_{h=0}^{H_k-1}
    \in
    \operatorname*{arg\,max}_{\{\pi_h\}_{h=0}^{H_k-1}}
    \mathbb{E}\!\left[
        \sum_{h=0}^{H_k-1}
        r_{\tau_k+h}\bigl(x_h,\pi_h(x_h)\bigr)
        \,\middle|\, x_0=s
    \right].
\end{equation}
Here $H_k:=\min\{\bar H,T-\tau_k\}$,
and the planning dynamics are $x_{h+1}
    \sim
    \hat{P}_{\tau_k-1}
    \bigl(\cdot\mid x_h,\pi_h(x_h)\bigr).$ The agent retains the first-step policy
$\pi_{\tau_k}^{\mathrm{alg}}
:=\pi_{k,0}^{\mathrm{plan}}$
for execution until the next update and resets its state estimate to $\hat{s}_{\tau_k}:=s_{\tau_k}$.

\item \textbf{Skip phase.}
Between completed updates, the agent reuses
$\pi_{\tau_k}^{\mathrm{alg}}$
over the execution interval $[\tau_k,\tau_{k+1})$
without replanning. At each time $t$ in this interval, it selects $a_t
    =
    \pi_{\tau_k}^{\mathrm{alg}}(\hat{s}_t).$ It maintains a
\emph{propagated maximum-a-posteriori} (propagated-MAP)
state estimate using
\[
    \hat{s}_{t+1}
    \in
    \operatorname*{arg\,max}_{s'\in\mathcal{S}}
    \hat{P}_{\tau_k-1}
    \bigl(s'\mid\hat{s}_t,a_t\bigr),
    \qquad
    t\in[\tau_k,\tau_{k+1}-1),
\]
with ties resolved using a fixed rule.
The state estimate used for control is reset only at completed
update times. Any observation collected in preparation for the
next update is retained for transition estimation without
resetting this control estimate before the update is completed.

\end{itemize}

\subsection{Performance Metric: Dynamic Regret}\label{sec:dynamic-regret}

To evaluate the performance of the skip-update algorithm, we compare its cumulative reward to that of an optimal policy with full knowledge of $\{P_t\}_{t\geq 0}$ and continuous state observations. For any $s\in\mathcal S$, we define the optimal cumulative reward and the skip-update cumulative reward as
{\footnotesize
\begin{align}
    J_T^\star(s)
    &=
    \max_{\{\pi_t\}_{t=0}^{T-1}}
    \mathbb{E}\!\left[
    \sum_{t=0}^{T-1}
    r_t\big(s_t,\pi_t(s_t)\big)
    \,\Big|\, s_0=s
    \right]
    \text{ and }
    J_T^{\mathrm{alg}}(s)
    &=
    \mathbb{E}\!\left[
    \sum_{t=0}^{T-1}
    r_t\big(s_t,\pi_t^{\mathrm{alg}}(\hat s_t)\big)
    \,\Big|\, s_0=s
    \right].
\end{align}
}
respectively. We measure performance of the skip-update algorithm via the dynamic regret
\begin{align}
    \mathcal{DR}(T)
    :=
    \max_{s\in\mathcal S}
    \big(
    J_T^\star(s)-J_T^{\mathrm{alg}}(s)
    \big).
\end{align}
The central question now becomes how the choice of update times impacts the dynamic regret, which we address in the next section through a dynamic regret analysis.
\section{Main Results}\label{sec:main_results}
We analyze the regret of the skip-update algorithm in
Section~\ref{sec:dynamic-regret-analysis} and use the analysis
to design an online, regret-guided update rule in
Section~\ref{sec:adaptive-rule}.

\subsection{Dynamic Regret Analysis}
\label{sec:dynamic-regret-analysis}

A key challenge in the analysis of the dynamic regret is controlling how
trajectory discrepancies under the skip-update algorithm's policy and
the optimal policy propagate over time. We control them through a
finite-time mixing condition, formalized via the following overlap
coefficient.

\begin{definition}[Overlap coefficient]\label{def:overlap}
Fix $m\ge1$. The overlap coefficient of two policies
$\{\pi^1_t\}_{t\ge0}$ and $\{\pi^2_t\}_{t\ge0}$
at time $t\ge0$ is defined as
\[
    \eta_t(\pi^1,\pi^2)
    :=
    \min_{s^1,s^2\in\mathcal{S}}
    \sum_{s'\in\mathcal{S}}
    \min\!\Big\{
        P^{\pi^1_{t:t+m-1}}_{t:t+m-1}(s'\mid s^1),
        P^{\pi^2_{t:t+m-1}}_{t:t+m-1}(s'\mid s^2)
    \Big\},
\]
where $P^{\pi_{t:t+m-1}}_{t:t+m-1}(s'\mid s)
    :=
    \sum_{s_1,\ldots,s_{m-1}\in\mathcal{S}}
    \prod_{k=0}^{m-1}
    P_{t+k}(s_{k+1}\mid s_k,\pi_{t+k}(s_k)),$
with $s_0=s$ and $s_m=s'$, is the $m$-step transition
kernel induced by the policy $\{\pi_t\}_{t\ge0}$
over $[t,t+m)$.
\end{definition}

\begin{assumption}[Uniform finite-time mixing~\citep{zhang2024predictive}]
\label{ass:mixing}
There exist $m\ge1$ and $\eta\in(0,1]$ such that
$\eta_t(\pi^{\mathrm{alg}},\pi^\star)\ge\eta$
for all $t\in\{0,\dots,T-1\}$, where
$\pi^\star=\{\pi^\star_t\}_{t\ge0}$
denotes the optimal policy of $\mathcal{M}$.
\end{assumption}

Assumption~\ref{ass:mixing} is a uniform overlap condition: the $m$-step distributions induced by the optimal policy and the skip-update policy share at least $\eta$ mass for any pair of starting states. Setting $\alpha:=1-\eta<1$, this overlap coefficient enters the span-seminorm error analysis over
$m$-step blocks. Similar uniform ergodicity assumptions appear
in~\citep{tsitsiklis1994asynchronous,li2019online,
yu2009online,zhang2024predictive}.
Next, we define the diameter of $\mathcal{M}$.

\begin{definition}[Diameter of $\mathcal{M}$~\citep{zhang2024predictive}]
\label{def:diameter}
Let $s^*_t$ denote the state achieving the highest expected
cumulative reward from time $t$ onward under the optimal policy.
The diameter of $\mathcal{M}$ is defined as
\[
    D
    :=
    \max_{s\in\mathcal{S}}
    \min_{\{\pi_h\}_{h\ge0}}
    \sup_t
    \mathbb{E}\!\left[
        \inf\{\tau>0:s_{t+\tau}=s^*_{t+\tau}\}
        \,\middle|\,s_t=s
    \right].
\]
\end{definition}

The diameter captures the worst-case expected time required
to reach the optimal state from any starting state across all
time steps. Under Assumption~\ref{ass:mixing},
$D$ is finite~\citep{zhang2024predictive}.

To state the bound, define the update-time kernel mismatch
$\hat\varepsilon_{t,i}
:=\max_{(s,a)}
\|P_{t+i}(\cdot\mid s,a)
-\hat P_{t-1}(\cdot\mid s,a)\|_{\mathrm{tv}}$,
where
$\|p-q\|_{\mathrm{tv}}
:=\tfrac12\sum_s|p(s)-q(s)|$
denotes the total variation distance,
the skip-interval mismatches
\[
    \bar\varepsilon_{\tau_{k(t)},t}
    :=
    \max_{(s,a)}
    \big\|
        P_t(\cdot\mid s,a)
        -
        P_{\tau_{k(t)}}(\cdot\mid s,a)
    \big\|_{\mathrm{tv}},
    \qquad
    \bar\delta_{\tau_{k(t)},t}
    :=
    \max_s
    \mathrm{sp}\big(
        r_t(s,\cdot)-r_{\tau_{k(t)}}(s,\cdot)
    \big),
\]
where
$\mathrm{sp}(v):=\max_a v(a)-\min_a v(a)$
denotes the span seminorm, and the aggregated per-step
skip mismatch
$\bar e_{\tau_{k(t)},j}
:=2\bar\varepsilon_{\tau_{k(t)},j}D
+\bar\delta_{\tau_{k(t)},j}$.

\begin{theorem}[Dynamic regret of the skip-update algorithm]
\label{th:main}
Under Assumption~\ref{ass:mixing}, the skip-update algorithm satisfies
\begin{equation}\label{eq:main_regret}
    \mathcal{DR}(T)
    \le
    \sum_{t\in\mathcal{T}_{\mathrm{upd}}}
    \mathcal{R}_{t}
    +
    \sum_{t\in\mathcal{T}_{\mathrm{skip}}}
    \Bigl(
        \mathcal{R}_{\tau_{k(t)}}
        +
        \alpha^{\left\lfloor\frac{T-t-1}{m}\right\rfloor}D
        +
        \bar e_{\tau_{k(t)},t}
        +
        \mathcal{E}^{\mathrm{drift}}_{\tau_{k(t)},t}
        +
        \mathcal{E}^{\mathrm{state}}_t
    \Bigr),
\end{equation}
where
\[
    \mathcal{R}_{\tau_{k(t)}}
    :=
    \alpha^{\left\lfloor\frac{H_{k(t)}-1}{m}\right\rfloor}D
    +
    2\hat\varepsilon_{\tau_{k(t)},0}\cdot D
    +
    \mathcal{E}^{\mathrm{est}}_{\tau_{k(t)}},
\]
and
$\mathcal{E}^{\mathrm{est}}_{\tau_{k(t)}}
=
\mathcal{O}\!\left(
    \frac{
        1-\alpha^{\left\lceil\frac{H_{k(t)}-1}{m}\right\rceil}
    }{1-\alpha}
    \max_{0\le i<H_{k(t)}}\hat\varepsilon_{\tau_{k(t)},i}\cdot mD
\right)$.
The
accumulated drift $\mathcal{E}^{\mathrm{drift}}_{\tau_{k(t)},t}$
and state-mismatch bound $\mathcal{E}^{\mathrm{state}}_t$
are defined in Appendix~\ref{app:main_proof}, together with
the full error expressions.
\end{theorem}

The proof follows the decomposition of the regret into update
and skip phases, adapting arguments
from~\citep{zhang2024predictive} for the update-time contribution
and developing new bounds for the skip-interval terms,
which is the main technical contribution of this work;
see Appendix~\ref{app:main_proof} for details.

\vspace{-0.5cm}
\paragraph{Interpretation.}
The term $\mathcal{R}_{\tau_{k(t)}}$ combines a finite-horizon
truncation error
$\alpha^{\lfloor(H_{k(t)}-1)/m\rfloor}D$,
an immediate estimation mismatch
$2\hat\varepsilon_{\tau_{k(t)},0}\cdot D$,
and an accumulated estimation error
$\mathcal{E}^{\mathrm{est}}_{\tau_{k(t)}}$
from planning with $\hat P_{\tau_{k(t)}-1}$.
The skip sum adds four terms:
an \emph{immediate drift} $\bar e_{\tau_{k(t)},t}$
since the most recent completed update,
a \emph{residual-horizon} term
$\alpha^{\lfloor(T-t-1)/m\rfloor}D$,
an \emph{accumulated drift}
$\mathcal{E}^{\mathrm{drift}}_{\tau_{k(t)},t}$
from temporal changes in the kernels and rewards,
and a \emph{state-mismatch} term $\mathcal{E}^{\mathrm{state}}_t$
from using the propagated-MAP state estimate.

\subsection{Adaptive Update Rule Design}\label{sec:adaptive-rule}

The goal is to allocate the available update budget to times
when continued execution without replanning is expected to incur
greater performance loss. Theorem~1 identifies three
skip-specific contributions: the residual-horizon term
$\alpha^{\lfloor(T-t-1)/m\rfloor}D$, the immediate drift
$\bar e_{\tau,t}$, and the accumulated drift
$\mathcal{E}_{\tau,t}^{\mathrm{drift}}$, where $\tau$ denotes the most
recent completed update time. These contributions motivate
the score
\begin{equation}\label{eq:adaptive-score}
\begin{aligned}
    S_t := {}&
    w_1\phi\!\left(\alpha^{\lfloor(T-t-1)/m\rfloor}D\right)
    + w_2\phi(\bar e_{\tau,t}) + w_3\phi\!\left(\mathcal{E}_{\tau,t}^{\mathrm{drift}}\right),
\end{aligned}
\end{equation}
where $w_1,w_2,w_3\ge0$ are hyperparameters and
$\phi(x):=\log(1+x)$ compresses the scales of the score
components. The implementation uses the drift-bound-based
proxies described in Appendix~C, with inputs available at the
current time. The score weights and threshold $\lambda$ are
selected using validation rollouts, as detailed in
Appendix~C.

\begin{wrapfigure}{r}{0.5\textwidth}
\begin{minipage}{0.5\textwidth}
\captionsetup{type=algorithm}
\hrule
\vspace{0.25em}
\caption{\small Budgeted Adaptive Update Rule}
\label{alg:adaptive_update}
\vspace{-0.6em}
\hrule
\vspace{0.2em}
\footnotesize
\begin{algorithmic}[1]

\REQUIRE Horizon $T$, budget $B$, planning horizon $H$,
         threshold $\lambda$, weights $w_1,w_2,w_3$

\STATE $\mathcal{T}_{\mathrm{upd}}\leftarrow\{\tau_0\}$
       where $\tau_0=1$

\STATE Observe $s_0$ and $s_1$,
       estimate $\widehat{P}_0$ from $(s_0,a_0,s_1)$,
       compute $\pi^{\mathrm{alg}}_{\tau_0}$,
       set $\hat{s}_1\leftarrow s_1$,
       $\tau\leftarrow1$, and $z_0\leftarrow\textsc{Skip}$

\FOR{$t=1,\ldots,T-1$}

    \IF{$z_{t-1}=\textsc{Update}$}

        \STATE $\mathcal{T}_{\mathrm{upd}}\leftarrow
               \mathcal{T}_{\mathrm{upd}}\cup\{t\}$

        \STATE Observe $s_t$, update the estimate
               $\widehat{P}_{t-1}$ using
               $(s_{t-1},a_{t-1},s_t)$,
               compute $\pi^{\mathrm{alg}}_t$,
               set $\hat{s}_t\leftarrow s_t$,
               and set $\tau\leftarrow t$

    \ENDIF

    \STATE $b_t\leftarrow B-|\mathcal{T}_{\mathrm{upd}}|$

    \IF{$b_t=0$ \textbf{or} $t=T-1$}

        \STATE $z_t\leftarrow\textsc{Skip}$

    \ELSIF{$b_t\ge T-1-t$}

        \STATE $z_t\leftarrow\textsc{Update}$

    \ELSE

        \STATE Evaluate $S_t$ via~\eqref{eq:adaptive-score}

        \STATE $L_t\leftarrow
               \left\lceil (T-1-t)/b_t \right\rceil$

        \IF{$S_t\ge\lambda$
            \textbf{or} $t+1-\tau\ge L_t$}

            \STATE $z_t\leftarrow\textsc{Update}$

        \ELSE

            \STATE $z_t\leftarrow\textsc{Skip}$

        \ENDIF

    \ENDIF

    \IF{$z_t=\textsc{Update}$}

        \STATE Record $s_t$, observing it if needed

    \ENDIF

    \STATE Select
           $a_t=\pi^{\mathrm{alg}}_{\tau}(\hat{s}_t)$
           and execute $a_t$

    \STATE Propagate $\hat{s}_{t+1}$ by the MAP transition
           under $\widehat{P}_{\tau-1}$

\ENDFOR
\end{algorithmic}
\vspace{0.2em}
\hrule
\end{minipage}
\vspace{-2.5em}
\end{wrapfigure}

An update trigger initiates transition acquisition; estimation
and replanning are completed one time step later.
Specifically, when the trigger fires at time $t$, the agent
observes and records $s_t$, then executes
$a_t=\pi_{\tau}^{\mathrm{alg}}(\hat s_t)$ using the policy
and control estimate maintained since the most recent completed
update. At time $t+1$, it observes $s_{t+1}$, incorporates
$(s_t,a_t,s_{t+1})$ into its dataset, estimates $\hat P_t$,
replans, and resets its control estimate to $s_{t+1}$.
Thus, the next completed update time is $\tau_{k+1}=t+1$.
The observation at the trigger time is retained for transition
estimation without resetting the control estimate.
The complete acquisition, estimation, and replanning operation
consumes one budget unit.

At each time $t$, any pending update is completed before
evaluating whether to initiate another update. Let
$b_t=B-|\mathcal T_{\mathrm{upd}}|$ denote the remaining
budget after this completion. A new update may be initiated
only if $b_t>0$ and $t\le T-2$, ensuring that it can be
completed within the horizon.
There are $N_t:=T-1-t$ remaining possible completion times, namely
$t+1,\ldots,T-1$. For $b_t>0$, define the pacing interval $L_t:=\left\lceil\frac{N_t}{b_t}\right\rceil.$ An update is initiated when $b_t\ge N_t
    \quad\text{or}\quad
    S_t\ge\lambda
    \quad\text{or}\quad
    t+1-\tau\ge L_t.$
The first condition forces updates when all remaining
completion times are needed to use the budget, while the
last condition limits the gap between completed updates.
For a feasible budget $1\le B\le T-1$, these conditions
ensure exact-budget usage, including the initial update
at $\tau_0=1$. Algorithm~\ref{alg:adaptive_update} summarizes the procedure.
A new acquisition may be initiated immediately after an update
is completed, allowing consecutive update times. In this case,
the current state observation is reused as the first endpoint
of the next transition.

\vspace{-0.5em}
\section{Experiments}
\label{sec:experiments}
\vspace{-0.5em}
We compare adaptive allocation (Algorithm~\ref{alg:adaptive_update}) against baselines with the same update budget and skip-interval execution semantics, isolating the effect of update timing.

\vspace{-0.75em}
\paragraph{Baselines.}
We compare against several exact-budget update schedules, grouped into two categories.
The first five are \emph{oblivious} to the drift signal and use the same finite horizon $T$: \textsc{Periodic} places the $B$ updates at uniform spacings over the horizon; \textsc{Best-offset Periodic} uses the same periodic spacing but chooses the phase/offset with the best validation performance; \textsc{Random} samples update times uniformly subject to the same budget; \textsc{Front-loaded} and \textsc{Back-loaded} spend the budget near the beginning and end of the horizon, respectively. The remaining three are \emph{drift-aware} but do not use the regret bound structure: \textsc{Drift-threshold} triggers updates when accumulated transition change exceeds a threshold; \textsc{Self-triggered Drift Forecast} chooses the next update time by forecasting when the accumulated drift since the current update will exceed a threshold; \textsc{Lazy Value-weighted Drift} updates only when the accumulated model drift is large on state--action pairs that are relevant under the currently reused policy. Further implementation details and hyperparameters are provided in Appendix~\ref{app:expt}.

\vspace{-0.75em}
\paragraph{Metrics and Protocol.}
The primary metric is empirical dynamic regret against a no-budget oracle with full knowledge of $\{P_t\}_{t\geq 0}$. All comparisons enforce exact-budget usage, including the initial update, ensuring performance differences reflect update timing alone. We also report update-placement diagnostics: precision, the fraction of non-initial updates placed inside high-drift intervals; coverage, the fraction of high-drift intervals that receive at least one update; and mean/max stale drift since the most recent update. Scores and baseline thresholds are selected on validation seeds and evaluated on disjoint held-out seeds; additional details are given in Appendix~\ref{app:expt}.

\vspace{-0.25em}

\begin{table*}[htp]
\centering
\scriptsize
\caption{\small Simulation results with propagated-MAP execution. Entries are empirical mean dynamic regret over 1000 held-out seeds; $n/T/H/B$ denotes grid size, horizon, planning horizon, and update budget. The adaptive-score column indicates which terms in the update score $S_t$ in~\eqref{eq:adaptive-score} are active: \textit{Gap} sets $w_2=0$, \textit{Drift} sets $w_3=0$, and \textit{Gap+Drift} sets $w_2,w_3\neq 0$.}
\label{tab:sim_all_baselines}
\vspace{0.35em}
\resizebox{\linewidth}{!}{%
\begin{tabular}{@{}lcl>{\columncolor{AdaptiveBlue}}ccccccccc@{}}
\toprule
Setting & $n/T/H/B$ & Adaptive score
& Adapt & Per. & Per.-BO & Rand. & Front & Back & Drift & Self & Lazy \\
\midrule
Piecewise-Late Burst
& 14/80/8/2
& Gap
& \textbf{1.91} & 4.93 & 2.78 & 4.45 & 5.48 & 4.93 & 3.87 & 3.87 & 3.45 \\
Piecewise-Two-Burst
& 12/50/7/4
& Gap
& \textbf{0.60} & 1.09 & 0.86 & 1.06 & 2.56 & 2.32 & 0.94 & 1.01 & 0.83 \\
Piecewise-Alternating
& 12/50/7/4
& Drift
& \textbf{0.58} & 0.87 & 0.79 & 1.01 & 2.51 & 2.26 & 0.89 & 0.81 & 0.90 \\
Sinusoidal-Medium
& 12/50/7/3
& Gap
& \textbf{0.23} & 0.57 & 0.61 & 0.65 & 1.92 & 2.13 & 0.41 & 0.41 & 1.28 \\
Piecewise-Narrow-Late
& 14/70/8/2
& Gap
& \textbf{0.83} & 2.24 & 0.92 & 2.04 & 2.59 & 2.24 & 0.92 & 0.92 & 1.10 \\
Piecewise-Sparse-Late
& 12/60/7/4
& Gap+Drift
& \textbf{0.51} & 0.78 & 0.78 & 0.79 & 2.08 & 1.82 & 0.58 & 0.65 & 0.74 \\
Piecewise-Sparse-Late
& 12/50/7/5
& Gap+Drift
& \textbf{0.34} & 0.40 & 0.44 & 0.53 & 1.68 & 1.41 & 0.42 & 0.54 & 0.58 \\
Piecewise-Three-Short
& 12/60/7/6
& Gap+Drift
& \textbf{0.56} & 0.61 & 0.89 & 1.19 & 3.50 & 3.30 & 0.68 & 0.74 & 1.21 \\
\bottomrule
\end{tabular}
}
\vspace{-0.5em}
\begin{flushleft}
\scriptsize
\textit{Abbreviations:} Adapt = proposed adaptive scheduler; Per. = uniform periodic; Per.-BO = best-offset periodic; Rand. = random exact-budget schedule; Drift = drift-threshold paced; Self = self-triggered drift forecast; Lazy = lazy value-weighted drift.
\end{flushleft}
\vspace{-1.5em}
\end{table*}

\subsection{Simulation Experiments}
\label{sec:simulation_setup}
\vspace{-0.25em}
We evaluate Algorithm~\ref{alg:adaptive_update} on the Mars-rover navigation scenario from Example~\ref{ex:mars_rover}, used here as a controlled testbed for studying update allocation under time-varying dynamics. The agent starts near one corner of an $n\times n$ grid and must navigate to a fixed goal in the opposite corner under stochastic, time-varying motion uncertainty, with rewards penalizing distance from the goal. We consider three drift categories: \emph{linear drift}, which produces gradual monotone change; \emph{sinusoidal drift}, which produces recurring nonuniform change; and \emph{piecewise drift}, following bounded-drift constructions used in TVMDP learning benchmarks~\citep{ornik2021learning}, which produces low-change intervals separated by high-change bursts. The setting names use a category--pattern convention: \textsc{Sinusoidal-Medium} denotes a sinusoidal drift profile with a medium period, while \textsc{Piecewise-Late}, \textsc{Piecewise-Two-Burst}, \textsc{Piecewise-Alternating}, and related names describe the timing and number of high-change intervals. Full environment parameters are listed in Appendix~\ref{app:expt}. All budgeted methods use the propagated-MAP skip semantics from Section~\ref{sec:skip-update}, with planning horizon $H$ at update times and the computed first-step policy reused until the next update. We compare all methods across different values of $n$, $T$, $H$, and $B$.

\begin{wraptable}{r}{0.45\textwidth}
\vspace{-1.5em}
\centering
\scriptsize
\caption{\small Update-placement diagnostics aggregated over the reported simulation settings. Precision and coverage are with respect to high-drift intervals; stale drift is accumulated transition change since the most recent update.}
\label{tab:update_diagnostics}
\vspace{0.35em}
\begin{tabular}{@{}lcccc@{}}
\toprule
Method & Prec. $\uparrow$ & Cov. $\uparrow$ & Mean stale $\downarrow$ & Max stale $\downarrow$ \\
\midrule
\rowcolor{AdaptiveBlue}
Adapt   & 0.455 & 0.417 & 0.069 & 0.198 \\
Per.    & 0.364 & 0.296 & 0.095 & 0.240 \\
Per.-BO & 0.364 & 0.294 & 0.074 & 0.208 \\
Rand.   & 0.484 & 0.328 & 0.091 & 0.250 \\
Front   & 0.045 & 0.013 & 0.177 & 0.400 \\
Back    & 0.182 & 0.054 & 0.168 & 0.401 \\
Drift   & 0.636 & 0.533 & 0.067 & 0.190 \\
Self    & 0.636 & 0.471 & 0.070 & 0.187 \\
Lazy    & 0.636 & 0.494 & 0.096 & 0.252 \\
\bottomrule
\end{tabular}
\vspace{-1em}
\end{wraptable}

\vspace{-0.5em}
\paragraph{Simulation Results and Discussions.}
Table~\ref{tab:sim_all_baselines} reports empirical dynamic regret across settings where updates are scarce relative to the horizon, so that each update affects an entire subsequent skip interval rather than only the current decision. Adaptive scheduling improves most clearly when transition changes are temporally localized: uniform schedules can spend updates during low-drift intervals, while purely drift-triggered schedules can update when the model changes even if replanning does not substantially alter the task-relevant decision. The active score terms are listed in Table~\ref{tab:sim_all_baselines}. Algorithm~\ref{alg:adaptive_update} triggers updates based on predicted cost of the current model-policy pair, not on elapsed time or transition signal magnitude. Table~\ref{tab:update_diagnostics} explains the update-allocation behavior behind the regret results. The adaptive scheduler does not place the largest fraction of its updates inside high-drift intervals: its precision is lower than the drift-aware schedules, yet its stale-drift values are comparable and its regret is lower. Together, these diagnostics show that the adaptive rule is not simply a peak detector for the drift signal. It keeps accumulated model staleness low while filtering updates through the estimated value of replanning. High transition change is a useful signal but not the objective; an update is valuable only when it improves the decision induced by the current state estimate and reused policy. 

\begin{wrapfigure}{r}{0.6\textwidth}
\vspace{-1.0em}
\centering
\includegraphics[width=\linewidth]{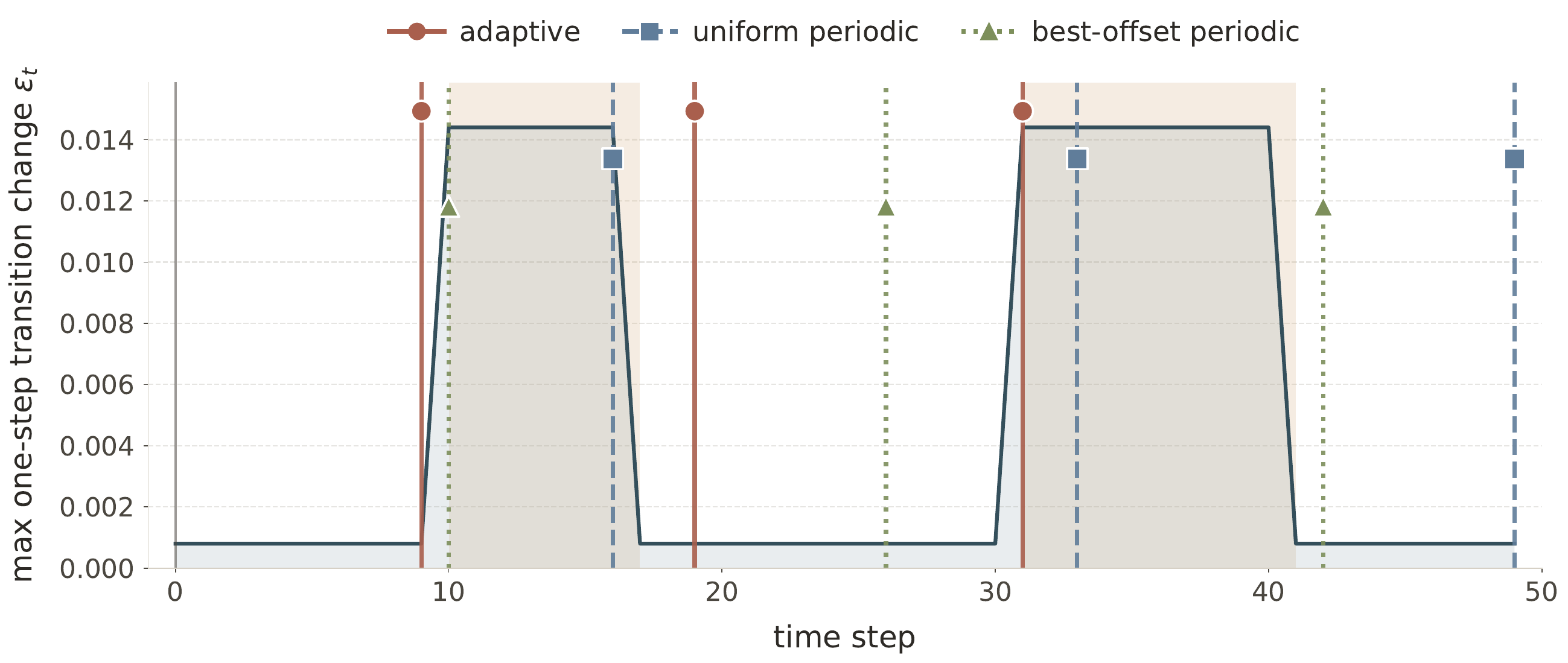}
\caption{\small Update timing in a representative piecewise-drift setting. 
The curve shows maximum one-step transition change, shaded regions denote high-drift 
intervals, and vertical markers show update times.}
\label{fig:update_timing}
\vspace{-1.0em}
\end{wrapfigure}
Figure~\ref{fig:update_timing} visualizes this for a representative piecewise-drift setting: the adaptive schedule concentrates updates around high-change portions of the horizon without simply tracking every drift peak, reflecting the joint influence of the current change signal, the remaining horizon, and the re-planning value. Figure~\ref{fig:sim_trajectories} shows the corresponding effect on realized navigation trajectories. The main-table simulations use simulator kernels at update times and exact current drift bounds. Appendix~\ref{app:expt} reports additional simulations using one globally frozen scheduler across settings, inflated and noisy conservative drift bounds, and comparisons against exhaustively optimal offline update schedules on small instances.

\begin{figure}[!h]
\vspace{-0.5em}
    \centering
    \includegraphics[width=0.8\linewidth]{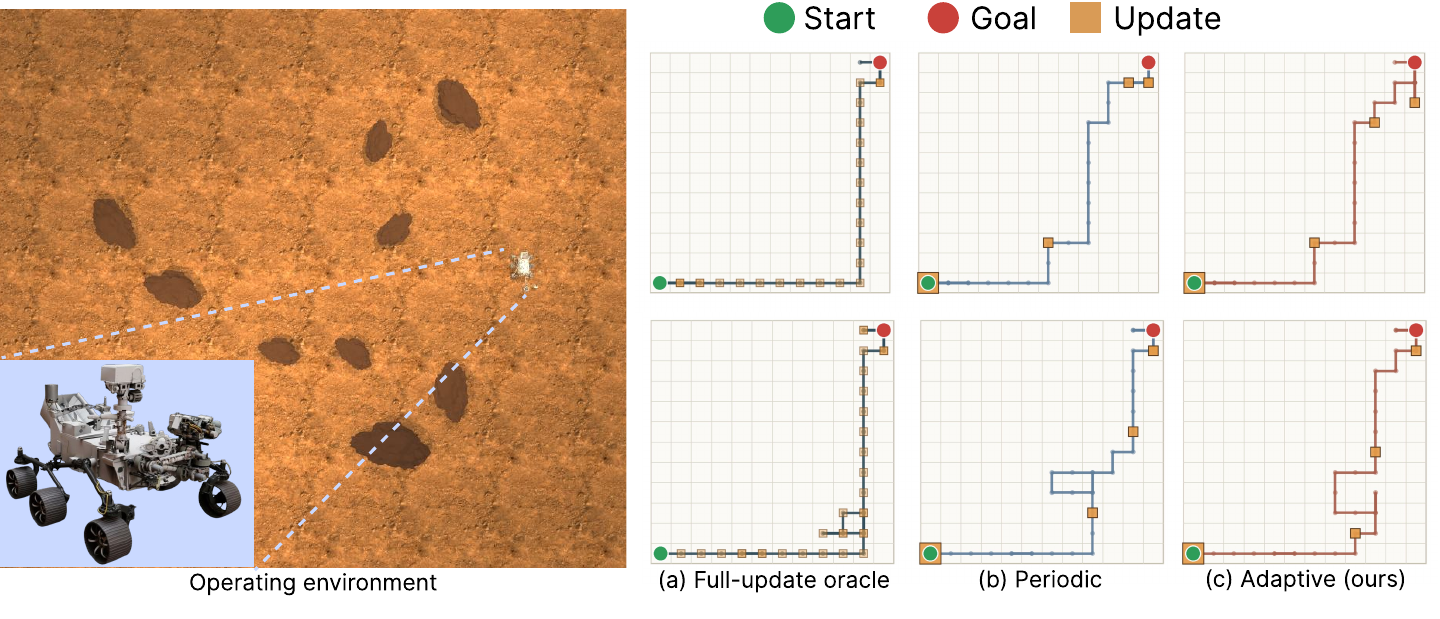}
    \caption{\small 
      Representative gridworld trajectories: (top) piecewise sparse-late drift setting with $n/T/H/B=12/50/7/5$; (bottom) three-short-burst drift setting with $n/T/H/B=12/60/7/6$. Adaptive allocation places updates near trajectory segments where stale policies would otherwise cause larger deviations, leading to a more direct path toward the goal.
      }
    \label{fig:sim_trajectories}
\vspace{-1.5em}
\end{figure}

\subsection{Hardware Experiments}
\label{sec:hardware_results}
\vspace{-0.5em}
\paragraph{Hardware Setup.}
We evaluate on a Crazyflie quadrotor in an indoor navigation task with three obstacle-density settings: no obstacles, sparse obstacles, and dense obstacles, each averaged over 15 independently generated layouts with the same start--goal structure. The high-level scheduler and planner run offboard, while the Crazyflie executes commanded motion using its low-level stabilization controller, matching the propagated-MAP semantics from Section~\ref{sec:skip-update}. The adaptive method uses the \textsc{gap+drift-log} update score. We follow the experimental setting of \cite{puthumanaillam2024weathering} and induce stochasticity by applying simulated wind disturbances that perturb the controller during execution.
\begin{table}[!h]
\vspace{-1.5em}
\centering
\scriptsize
\begingroup
\ElegantTableSetup
\setlength{\tabcolsep}{5.2pt}
\caption{\small Crazyflie hardware results across obstacle densities. Each setting averages over 15 layouts; all methods use $n/T/H/B=12/60/7/6$. Success is the fraction of layouts reaching the goal, collision is the avg.\ number of collision events per layout, and cost is normalized trajectory cost with collision penalties.}
\label{tab:crazyflie_results}
\vspace{0.35em}
\begin{adjustbox}{max width=\textwidth,center}
\begin{tabular}{@{}lccccccccc@{}}
\toprule
&
\multicolumn{3}{c}{No obstacles}
& \multicolumn{3}{c}{Sparse obstacles}
& \multicolumn{3}{c}{Dense obstacles} \\
\cmidrule(lr){2-4}
\cmidrule(lr){5-7}
\cmidrule(l){8-10}
Method
& Success $\uparrow$ & Collision $\downarrow$ & Cost $\downarrow$
& Success $\uparrow$ & Collision $\downarrow$ & Cost $\downarrow$
& Success $\uparrow$ & Collision $\downarrow$ & Cost $\downarrow$ \\
\midrule
\rowcolor{AdaptiveBlue}
Adapt
& 0.93 & 0.00 & 1.07
& \textbf{0.87} & \textbf{0.20} & \textbf{1.42}
& \textbf{0.73} & \textbf{0.47} & \textbf{2.05} \\

Per.
& 0.93 & 0.00 & 1.08
& 0.73 & 0.40 & 1.72
& 0.53 & 0.87 & 2.70 \\

Per.-BO
& \textbf{1.00} & 0.00 & \textbf{1.05}
& 0.80 & 0.33 & 1.58
& 0.60 & 0.73 & 2.42 \\

Rand.
& 0.87 & 0.00 & 1.13
& 0.67 & 0.53 & 1.84
& 0.47 & 1.00 & 2.85 \\

Front
& 0.87 & 0.00 & 1.18
& 0.60 & 0.67 & 2.05
& 0.40 & 1.27 & 3.22 \\

Back
& 0.87 & 0.00 & 1.17
& 0.60 & 0.73 & 2.12
& 0.40 & 1.33 & 3.35 \\

Drift
& 0.93 & 0.00 & 1.09
& 0.80 & 0.27 & 1.53
& 0.67 & 0.60 & 2.20 \\

Self
& 0.93 & 0.00 & 1.10
& 0.80 & 0.33 & 1.56
& 0.67 & 0.60 & 2.25 \\

Lazy
& 0.93 & 0.00 & 1.11
& 0.73 & 0.40 & 1.69
& 0.60 & 0.73 & 2.38 \\
\bottomrule
\end{tabular}
\end{adjustbox}
\vspace{-0.5em}
\begin{flushleft}
\scriptsize
\textit{Abbreviations:} Adapt = proposed adaptive scheduler; Per. = uniform periodic; Per.-BO = best-offset periodic; Rand. = random exact-budget schedule; Drift = drift-threshold paced; Self = self-triggered drift forecast; Lazy = lazy value-weighted drift.
\end{flushleft}
\endgroup
\vspace{-2.5em}
\end{table}
\vspace{-1.0em}
\paragraph{Hardware Results and Discussions.}
Table~\ref{tab:crazyflie_results} evaluates the same budgeted-update mechanism on a physical quadrotor. The no-obstacle setting is largely insensitive to update timing: most schedules complete the task reliably because the nominal path is simple and re-planning is rarely safety-critical. The sparse and dense settings are more informative. As obstacle density increases, stale plans are more likely to produce inefficient motion or unsafe corrections, and the benefit of allocating updates at task-relevant times becomes more visible. The adaptive scheduler gives the strongest overall tradeoff in the obstacle settings: it maintains higher success while reducing collision counts and trajectory cost under the same update budget.
\begin{wrapfigure}{r}{0.6\textwidth}
\vspace{-0.25em}
\centering
\includegraphics[width=\linewidth]{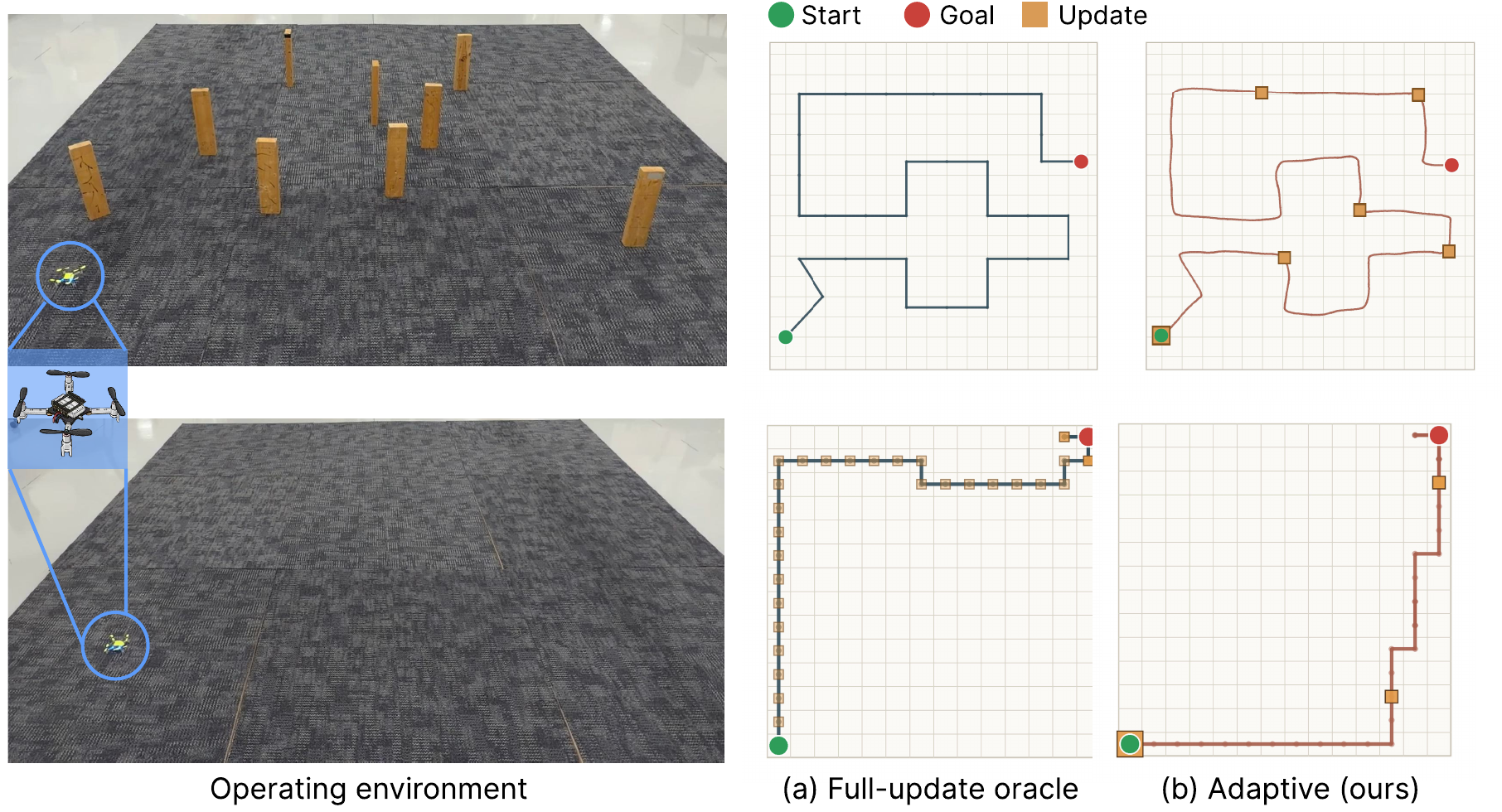}
\caption{\small Representative trajectories with $n/T/H/B=12/60/7/6$: (top) dense-obstacle navigation setting; (bottom) sparse-obstacle indoor navigation setting.}
\label{fig:crazyflie_trajectories}
\vspace{-2.0em}
\end{wrapfigure}
Drift-aware schedules remain competitive, which is expected because obstacle-rich environments create localized changes where model staleness matters. However, the adaptive rule improves on these schedules by coupling model staleness with the expected value of replanning, rather than triggering solely from transition change. Figure~\ref{fig:crazyflie_trajectories} shows representative trajectories for the obstacle-density settings and marks the update locations along the executed paths.

\section{Concluding Remarks}\label{sec:conclusion}
\vspace{-0.75em}
We proposed an update-scheduling rule for TVMDPs under a budget on state observation and re-planning, derived directly from a dynamic regret analysis of a maximum-likelihood-based skip-update algorithm. The rule distills the skip-interval contribution of the regret bound into a lightweight online score, computable from quantities the agent already maintains, and allocates the available budget adaptively without recomputing value functions. Simulation experiments on Mars-rover navigation and hardware experiments on a Crazyflie quadrotor demonstrate that the proposed adaptive rule consistently outperforms budgeted baselines across all tested settings.
\vspace{-0.77em}
\paragraph{Limitations.}
The current framework operates under several assumptions that define its scope. 
The drift bounds $\{\varepsilon_t\}_{t\geq 0}$ are assumed known and 
state-independent; relaxing either --- by estimating $\varepsilon_t$ online 
or allowing spatially heterogeneous drift --- would broaden applicability. 
Relatedly, the regret-guided adaptive rule relies on $\{\varepsilon_t\}_{t\geq 0}$ 
to compute the score $S_t$; if these bounds are overly conservative or fail to 
reflect the true temporal profile of the drift, the score may trigger updates 
at suboptimal times and the adaptive rule may not outperform simpler schedules 
such as periodic allocation, making online estimation of $\varepsilon_t$ an 
important direction for future work. The regret bound, derived via contraction 
on the span seminorm of value-function differences, captures worst-case 
deviation between the optimal and skip-update policies; when the propagated 
state remains close to the true state, this bound can be conservative. 
Similarly, the adaptive score targets the worst-case regret contribution of 
each skip interval without incorporating rewards accumulated along the traverse, 
leaving room for reward-aware scheduling. Finally, replacing the propagated-MAP 
point estimate with a belief-state representation would extend the framework to 
partially observable settings; however, this direction is substantially more 
challenging, as it would require redefining the regret decomposition over the 
belief simplex, extending the mixing assumption to belief-space dynamics, and 
managing the additional computational cost of belief updates within the budgeted 
setting --- each of which represents a non-trivial departure from the current 
framework.
\section*{Acknowledgment}
This work was supported by the Office of Naval Research under grant no. N00014-25-1-2369.

\bibliography{references}






\begin{appendix}

\section*{\centering Appendix}
\section*{Table of Contents}
\begin{itemize}
    \item Appendix~\ref{app:skip_update} provides additional details on the skip-update algorithm.
    \item Appendix~\ref{app:main_proof} contains the proof of Theorem~\ref{th:main}.
    \item Appendix~\ref{app:expt} provides additional details on the implementation of the adaptive rule (Algorithm~\ref{alg:adaptive_update}) and the experimental setup.
\end{itemize}

\section{Additional Details for the Skip-update Algorithm}\label{app:skip_update}
In this appendix, we provide the constrained maximum likelihood estimation procedure from the update phase of the skip-update algorithm from Section~\ref{sec:skip-update}. The presentation adapts the formulation in~\citep{ornik2021learning} to the budgeted update setting considered here.

At time $t \geq 1$, the agent has access to the dataset
\[
    \mathcal{D}_t
    =
    \{(s_{\tau_j-1}, a_{\tau_j-1}, s_{\tau_j}) : 0 \le j \le k(t)\},
\]
collected through completed update events, where $\tau_0 = 1$ denotes the first completed update and $(s_0, a_0, s_1)$ is the initialization transition. Under a sequence of transition kernels $\{\hat{P}_{\tau_k-1}\}_{k=0}^{k(t)}$, the log-likelihood of the observed data $\mathcal{D}_{t}$ is
\[
    \ell\big(\{\hat{P}_{\tau_k-1}\}_{k=0}^{k(t)}; \mathcal{D}_t\big)
    =
    \sum_{k=0}^{k(t)}
    \log \hat{P}_{\tau_k-1}(s_{\tau_k} \mid s_{\tau_k-1}, a_{\tau_k-1}).
\]
At each update time $t\in\mathcal{T}_{\mathrm{upd}}$, the maximum likelihood estimates of the transition kernels $\{\hat{P}_{\tau_k-1}\}_{k=0}^{k(t)}$ are obtained by solving the following constrained optimization problem
\begin{equation}\label{eq:ccmle}
\begin{aligned}
    \max_{\{\hat{P}_{\tau_k-1}\}_{k=0}^{k(t)}} \quad
    & \sum_{k=0}^{k(t)}
    \log \hat{P}_{\tau_k-1}(s_{\tau_k} \mid s_{\tau_k-1}, a_{\tau_k-1}) \\
    \text{s.t.}\quad
    & \sum_{s'\in\mathcal{S}} \hat{P}_{\tau_k-1}(s' \mid s, a) = 1,
    && \forall s,a,\; 0 \le k \le k(t), \\
    & \hat{P}_{\tau_k-1}(s' \mid s, a) \ge 0,
    && \forall s,s',a,\; 0 \le k \le k(t), \\
    & \big|\hat{P}_{\tau_{k+1}-1}(s' \mid s, a)
          - \hat{P}_{\tau_{k}-1}(s' \mid s, a)\big|
    \le \sum_{\ell=\tau_{k}-1}^{\tau_{k+1}-2} \varepsilon_{\ell},
    && \forall s,s',a,\; 0 \le k < k(t).
\end{aligned}
\end{equation}
The constraints ensure that each $\hat{P}_{\tau_k-1}$ is a valid transition kernel and that successive kernels satisfy the bounded drift condition.

\section{Proof of Theorem~\ref{th:main} and Supporting Results}
\label{app:main_proof}
This appendix provides the proof of Theorem~\ref{th:main}.
We begin with notation and value function definitions
below, followed by a helper lemma
on multi-stage error propagation (Section~\ref{app:helper}),
the main proof (Section~\ref{app_1}), and the supporting
lemmas for update times (Section~\ref{app_2}) and skip
intervals (Section~\ref{app_3}).

\subsection*{Notation and Definitions}
\label{app:notation}

We briefly collect the notation used throughout this appendix.
We denote by $\mathbb{N}$ the set of nonnegative integers and by $\mathbb{R}$ the set of real numbers.
For a finite set $\mathcal{X}$, let $\Delta(\mathcal{X})$ denote the probability simplex over $\mathcal{X}$.
A controlled transition kernel is a mapping $P : \mathcal{S} \times \mathcal{A} \to \Delta(\mathcal{S})$,
where $P(\cdot \mid s,a)$ denotes the distribution of the next state given state $s$ and action $a$;
we write $P_t(\cdot \mid s,a)$ for the kernel at time $t$.
For a controlled transition kernel $P$ and $f : \mathcal{S} \to \mathbb{R}$, define
$(Pf)(s,a) := \mathbb{E}_{s' \sim P(\cdot \mid s,a)}[f(s')].$
For $f : \mathcal{S} \to \mathbb{R}$, the span seminorm is
$\mathrm{sp}(f) := \max_{s \in \mathcal{S}} f(s) - \min_{s \in \mathcal{S}} f(s),$
and for $\mu, \nu \in \Delta(\mathcal{S})$, the total variation distance is
$\|\mu - \nu\|_{\mathrm{tv}} := \tfrac{1}{2} \sum_{s \in \mathcal{S}} |\mu(s) - \nu(s)|.$
A time-varying policy sequence is a collection $\{\pi_t\}_{t=0}^{T-1}$ with $\pi_t : \mathcal{S} \to \mathcal{A}$. We will repeatedly use the following standard bound: for any two controlled
transition kernels $P, \hat{P} : \mathcal{S} \times \mathcal{A} \to \Delta(\mathcal{S})$
and any function $f : \mathcal{S} \to \mathbb{R}$,
\begin{equation}\label{eq:tv_span_bound}
    \mathrm{sp} \Big( \mathbb{E}_{s'\sim P(\cdot\mid s,\cdot)}[f(s')]
    - \mathbb{E}_{s'\sim \hat{P}(\cdot\mid s,\cdot)}[f(s')] \Big)
    \leq 2\max_{a\in\mathcal{A}}\|P(\cdot\mid s,a) - \hat{P}(\cdot\mid s,a)\|_{\mathrm{tv}}\,\mathrm{sp}(f).
\end{equation}

\paragraph{Value Functions.}
The optimal value function of the TVMDP $\mathcal{M}$ at time $t$ is
\[
    V_t^{\star}(s)
    :=
    \max_{\{\pi_k\}_{k=t}^{T-1}}
    \mathbb{E}\!\left[
        \sum_{k=t}^{T-1} r_k\big(s_k,\pi_k(s_k)\big)
        \,\mid\, s_t = s
    \right], \qquad \forall s \in \mathcal{S},
\]
with $V_T^\star \equiv 0$, satisfying the Bellman recursion
\[
    V_t^{\star}(s)
    =
    \max_{a \in \mathcal{A}}
    \left\{
        r_t(s,a)
        +
        \mathbb{E}_{s' \sim P_t(\cdot \mid s,a)}
        \big[ V_{t+1}^{\star}(s') \big]
    \right\}, \qquad t = 0,\dots,T-1.
\]
The corresponding optimal $Q$-value function at time $t$ is
\[
    Q_t^\star(s,a) = r_t(s,a)+\mathbb{E}_{s' \sim P_t(\cdot \mid s,a)}[V_{t+1}^{\star}(s')],
    \qquad \forall (s,a) \in \mathcal{S} \times \mathcal{A},
\]
so that $V_t^\star(s) = \max_{a \in \mathcal{A}} Q_t^\star(s,a)$.

The value function associated with the skip-update algorithm is
\[
    V_t^{\mathrm{alg}}(s)
    :=
    \mathbb{E}\!\left[
        \sum_{\ell=t}^{T-1} r_{\ell}\big(s_{\ell},\pi_{\ell}^{\mathrm{alg}}(\hat{s}_{\ell})\big)
        \,\mid\, s_t = s
    \right], \qquad t=0,\dots,T-1,
\]
where $\hat{s}_{\ell}$ is the propagated-MAP state estimate at time $\ell$ and
$\pi_{\ell}^{\mathrm{alg}}=\pi_{\tau_{k(\ell)}}^{\mathrm{alg}}$ for $\ell\ge1$.
At initialization, $\hat{s}_0=s_0$ and $\pi_0^{\mathrm{alg}}(\hat{s}_0)=a_0$.

At each update time $\tau_k$, the policy $\pi^{\mathrm{alg}}_{\tau_k}$ is obtained
by solving a finite-horizon auxiliary MDP with estimated kernel
$\hat{P}_{\tau_k-1}$, rewards $\{r_{\tau_k+h}\}_{h=0}^{H_k-1}$, and horizon
$H_k$, as in~\eqref{eq:policy-opt-finite}. We denote the optimal value function
of this auxiliary MDP by $\hat{W}^{\star}_{t,k}$, defined for $t \in [\tau_k,
\tau_k + H_k]$ as
\[
    \hat{W}^{\star}_{t,k}(s)
    :=
    \max_{\{\pi_h\}_{h=0}^{\tau_k+H_k-t-1}}
    \mathbb{E}\!\left[
        \sum_{h=0}^{\tau_k+H_k-t-1} r_{t+h}\big(x_h,\pi_h(x_h)\big)
        \,\mid\, x_0 = s
    \right],
\]
with $x_{h+1} \sim \hat{P}_{\tau_k-1}(\cdot \mid x_h, \pi_h(x_h))$ and
$\hat{W}^{\star}_{\tau_k+H_k,k} \equiv 0$.

\subsection{Helper Lemma: Multi-stage Error Propagation}
\label{app:helper}
We state a helper lemma that quantifies how the difference between the optimal value functions of two TVMDPs propagates over multiple stages, under the mixing assumption on one of them. This result is taken directly from~\citep{zhang2024predictive}, and will be used to control the error between value functions associated with different models in Sections~\ref{app_1} and~\ref{app_2}.

\medskip
\noindent
Let $(P_t, r_t)$ and $(\tilde{P}_t, \tilde{r}_t)$ denote the transition kernels and reward functions of two TVMDPs defined on the same state-action spaces $\mathcal{S} \times \mathcal{A}$. Let the one-stage optimal Bellman 
operators $\mathcal{B}_t$ and $\tilde{\mathcal{B}}_t$ be defined as
\begin{align*}
    (\mathcal{B}_t f)(s)
    &:=
    \max_{a \in \mathcal{A}}
    \Big\{
        r_t(s,a)
        +
        \mathbb{E}_{s' \sim P_t(\cdot \mid s,a)}[f(s')]
    \Big\},\\
    (\tilde{\mathcal{B}}_t f)(s)
    &:=
    \max_{a \in \mathcal{A}}
    \Big\{
        \tilde{r}_t(s,a)
        +
        \mathbb{E}_{s' \sim \tilde{P}_t(\cdot \mid s,a)}[f(s')]
    \Big\},
\end{align*}
for all $s \in \mathcal{S}$ and any $f:\mathcal{S} \rightarrow \mathbb{R}$. 
The corresponding optimal value functions $\{V_t^\star\}_{t=0}^{T}$ and 
$\{\tilde{V}_t^\star\}_{t=0}^{T}$ satisfy
\[
    V_t^\star = \mathcal{B}_t V_{t+1}^\star,
    \qquad
    \tilde{V}_t^\star = \tilde{\mathcal{B}}_t \tilde{V}_{t+1}^\star,
    \qquad t = 0,\dots,T-1,
\]
with $V_T^\star \equiv 0$ and $\tilde{V}_T^\star \equiv 0$.

\begin{lemma}[Multi-stage error propagation under mixing {\citep{zhang2024predictive}}]
\label{lem:multi_stage_error}
Suppose the TVMDP $(P_t, r_t)$ satisfies Assumption~\ref{ass:mixing} with 
constants $m$ and $\eta$, and define $\alpha = 1 - \eta$. Then, for any 
$t \geq 0$ and any integer $N \geq 1$ with $N \leq T - t$, the following holds:
\begin{align*}
    \mathrm{sp} \big( V^{\star}_{t} - \tilde{V}^{\star}_{t} \big)
    &\leq
    \alpha^{\lfloor \frac{N}{m} \rfloor} \ \mathrm{sp} \Big( V^{\star}_{t+N} - \tilde{V}^{\star}_{t+N} \Big)\\
    &\quad + 2 \alpha^{\lfloor \frac{N}{m} \rfloor}
    \sum_{i=1}^{N-\lfloor \frac{N}{m} \rfloor m}
    \Big(
        \tilde{\varepsilon}_{t+\lfloor \frac{N}{m} \rfloor m+i-1}\
        \mathrm{sp} \big( \tilde{V}^{\star}_{t+\lfloor \frac{N}{m} \rfloor m+i} \big)
        + \tilde{\delta}_{t+\lfloor \frac{N}{m} \rfloor m+i-1}
    \Big)\\
    &\quad + 2 \sum_{\ell = 0}^{\lfloor \frac{N}{m} \rfloor - 1}
    \alpha^{\ell}
    \sum_{i=1}^{m}
    \Big(
        \tilde{\varepsilon}_{t+\ell m+i-1}\
        \mathrm{sp} \big( \tilde{V}^{\star}_{t+\ell m+i} \big)
        + \tilde{\delta}_{t+\ell m+i-1}
    \Big).
\end{align*}
Here,
\[
\tilde{\varepsilon}_{k} = \max_{(s,a)} \|P_{k}(\cdot \mid s,a)-\tilde{P}_{k}(\cdot \mid s,a)\|_{\operatorname{tv}},
\qquad
\tilde{\delta}_{k} = \max_{(s,a)} |r_{k}(s,a)-\tilde{r}_{k}(s,a)|.
\]
\end{lemma}
\noindent
This lemma shows that the span of the value function difference contracts 
geometrically over blocks of length $m$ — driven by the mixing of $(P_t, r_t)$ 
— up to additive terms capturing discrepancies in transition kernels and rewards 
between the two TVMDPs. The proof follows the block-contraction argument in~\citep{zhang2024predictive} and is omitted here.

\subsection{Proof of Theorem~\ref{th:main}}
\label{app_1}
We now prove Theorem~\ref{th:main} by decomposing the regret into update-time 
and skip-interval contributions. Recall the value functions $V_t^\star$, 
$Q_t^\star$, and $V_t^{\mathrm{alg}}$ defined in the notation and definitions above.

\medskip
With a slight abuse of notation, let us denote
\begin{align*}
    &a_{t}^\star = \pi_{t}^{\star}(s_{t}) = \arg\max_{a}Q_{t}^\star(s_{t},a),\\
    &a_{t}^{\mathrm{alg}} = \pi_{t}^{\mathrm{alg}}(\hat{s}_{t}) 
    = \pi_{\tau_{k(t)}}^{\mathrm{alg}}(\hat{s}_{t}),
    \qquad t=1,\dots,T-1.
\end{align*}
At $t=0$, $a_0^{\mathrm{alg}}=a_0$ is the initialization action.
Since $\tau_0=1$, we use
$\mathcal{T}_{\mathrm{skip}}=\{1,\dots,T-1\}\setminus\mathcal{T}_{\mathrm{upd}}$.

Hence, the dynamic regret can be expressed as
\[
    \mathcal{DR}(T)=\max_{s_0\in\mathcal{S}}\Big(V_0^\star(s_0)-V_0^{\mathrm{alg}}(s_0)\Big).
\]

Using the definition of the value function and adding and subtracting 
$\mathbb{E}_{s_{1} \sim P_0(\cdot \mid s_0,a_{0}^{\mathrm{alg}})} \big[ V_{1}^{\star}(s_{1}) \big]$, 
we can rewrite $V_0^\star(s_0)-V_0^{\mathrm{alg}}(s_0)$ as
\begin{align*}
    &V_0^\star(s_0)-V_0^{\mathrm{alg}}(s_0)\\
    &\quad = 
    r_0(s_{0},a_{0}^{\star})
    +
    \mathbb{E}_{s_{1} \sim P_0(\cdot \mid s_0,a_0^\star)}
    \big[ V_{1}^{\star}(s_{1}) \big] - r_0(s_{0},a_{0}^{\mathrm{alg}})
    -
    \mathbb{E}_{s_{1} \sim P_0(\cdot \mid s_0,a_{0}^{\mathrm{alg}})}
    \big[ V_{1}^{\mathrm{alg}}(s_{1}) \big]\\
    &\qquad + \mathbb{E}_{s_{1} \sim P_0(\cdot \mid s_0,a_{0}^{\mathrm{alg}})}
    \big[ V_{1}^{\star}(s_{1}) \big] - \mathbb{E}_{s_{1} \sim P_0(\cdot \mid s_0,a_{0}^{\mathrm{alg}})}
    \big[ V_{1}^{\star}(s_{1}) \big]\\
    &\quad = 
    Q_0^\star(s_0,a_0^\star)-Q_0^\star(s_0,a_0^{\mathrm{alg}})
    + \mathbb{E}_{s_1\sim P_0(\cdot\mid s_0,a_0^{\mathrm{alg}})}
    \big[V_1^\star(s_1)-V_1^{\mathrm{alg}}(s_1)\big].
\end{align*}

Then, by iterating for $t=0,\dots,T-1$, we obtain
\begin{equation}\label{eq:V0_gap}
\begin{aligned}
    &V_0^{\star}(s_0) - V_0^{\mathrm{alg}}(s_0)\\
    &\quad =
    Q_{0}^{\star}(s_{0},a_{0}^\star) - Q_{0}^{\star}(s_{0},a_{0}^{\mathrm{alg}})\\
    &\qquad + \sum_{k=1}^{T-1} \mathbb{E}_{s_{1}\sim P_{0}(\cdot \mid s_{0},a_{0}^{\mathrm{alg}})}\ 
    \dots\ 
    \mathbb{E}_{s_{k}\sim P_{k-1}(\cdot \mid s_{k-1},a_{k-1}^{\mathrm{alg}})} 
    \!\big[ Q_{k}^{\star}(s_{k},a_{k}^\star) - Q_{k}^{\star}(s_{k},a_{k}^{\mathrm{alg}}) \big].
\end{aligned}
\end{equation}

\medskip
Now define $\Delta_t(s_{0})$ as
\begin{equation}\label{eq:per-step-regret}
    \Delta_t(s_{0}) := \mathbb{E}_{s_{1}\sim P_{0}(\cdot \mid s_{0},a_{0}^{\mathrm{alg}})}\ 
    \dots\ 
    \mathbb{E}_{s_{t}\sim P_{t-1}(\cdot \mid s_{t-1},a_{t-1}^{\mathrm{alg}})} 
    \!\big[ Q_{t}^{\star}(s_{t},a_{t}^\star) - Q_{t}^{\star}(s_{t},a_{t}^{\mathrm{alg}}) \big],
    \qquad t=1,\dots,T-1,
\end{equation}
and $\Delta_{0}(s_{0}) := Q_{0}^{\star}(s_{0},a_{0}^\star) - Q_{0}^{\star}(s_{0},a_{0}^{\mathrm{alg}})$. 
Then, the dynamic regret can be written as
\begin{equation}
    \mathcal{DR}(T)
    =
    \max_{s_{0} \in \mathcal{S}}\ \sum_{t=0}^{T-1} \Delta_t(s_{0})
    =
    \max_{s_{0} \in \mathcal{S}}\Bigg\{
    \underbrace{\sum_{t\in\mathcal{T}_{\mathrm{upd}}} \Delta_t(s_{0})}_{\text{update steps}}
    +
    \underbrace{\sum_{t\in\mathcal{T}_{\mathrm{skip}}} \Delta_t(s_{0})}_{\text{skip steps}}
    \Bigg\}.
\label{eq:two-way-decomp}
\end{equation}

\medskip
Next, we state two lemmas: one for bounding the optimal $Q$-value difference 
at update times, i.e.\ $\Delta_{t}(s_{0})$ for $t \in \mathcal{T}_{\mathrm{upd}}$, 
and another for bounding the optimal $Q$-value difference at skip times, 
i.e.\ $\Delta_{t}(s_{0})$ for $t \in \mathcal{T}_{\mathrm{skip}}$.

\begin{lemma}\label{lem:Q_value_diff_obs}
Suppose $t \in \mathcal{T}_{\mathrm{upd}}$. Then $\Delta_{t}(s_{0})$ can be 
bounded as:
\begin{align*}
    \Delta_{t}(s_{0})
    &\leq 
    \alpha^{\lfloor \frac{H_{k(t)}-1}{m} \rfloor}\ D 
    + 2\hat{\varepsilon}_{t,0}\ D 
    + 2 \sum_{\ell = 0}^{\lfloor \frac{H_{k(t)}-1}{m} \rfloor - 1} \alpha^{\ell} 
    \sum_{i=1}^{m} \hat{\varepsilon}_{t,\ell m+i}\ D\\
    &\quad + 2 \alpha^{\lfloor \frac{H_{k(t)}-1}{m} \rfloor} 
    \sum_{i=1}^{H_{k(t)}-\lfloor \frac{H_{k(t)}-1}{m} \rfloor m - 1} 
    \hat{\varepsilon}_{t,\lfloor \frac{H_{k(t)}-1}{m} \rfloor m+i}\ D,
\end{align*}
where $\hat{\varepsilon}_{t,i} := \max_{(s,a)} \|P_{t+i}(\cdot \mid s,a)
-\hat{P}_{t-1}(\cdot \mid s,a)\|_{\operatorname{tv}}$ and $D$ is the diameter 
of $\mathcal{M}$ defined in Definition~\ref{def:diameter}.
The bound uses $\mathrm{sp}(V_u^\star) \leq D$ and
$\mathrm{sp}(\hat{W}_{u,k}^\star) \leq D$ for all admissible times $u$
and update indices $k$, together with the block-contraction condition
required by Lemma~\ref{lem:multi_stage_error}.
\end{lemma}

\medskip
The proof of this lemma is provided in Appendix~\ref{app_2}. We now present 
the following result, which constitutes the main technical contribution.

\begin{lemma}\label{lem:Q_value_diff_skip}
Suppose $t \in \mathcal{T}_{\mathrm{skip}}$ and $\tau_{k(t)}$ is the most recent update time. 
Then $\Delta_{t}(s_{0})$ can be bounded as:
\begin{align*}
    \Delta_{t}(s_{0})
    &\leq
    \mathcal{R}_{\tau_{k(t)}}
    + \alpha^{\lfloor \frac{T-t-1}{m}\rfloor}\ D
    + \bar{e}_{\tau_{k(t)},t}
    + \mathcal{E}^{\mathrm{drift}}_{\tau_{k(t)},t}
    + \mathcal{E}^{\mathrm{state}}_{t},
\end{align*}
where
\begin{align*}
    \bar{e}_{\tau_{k(t)},t}
    &:= 2\bar{\varepsilon}_{\tau_{k(t)},t}\ D
    + \bar{\delta}_{\tau_{k(t)},t},\\
    \mathcal{E}^{\mathrm{drift}}_{\tau_{k(t)},t}
    &:= 2\alpha^{\lfloor \frac{T-t-1}{m}\rfloor}\
    \sum_{i=0}^{T-t-2-\lfloor\frac{T-t-1}{m}\rfloor m}
    \Big(
        \tilde{\varepsilon}_{\tau_{k(t)}+1+\lfloor \frac{T-t-1}{m}\rfloor m+i}\ D
        + \tilde{\delta}_{\tau_{k(t)}+1+\lfloor \frac{T-t-1}{m}\rfloor m+i}
    \Big)\\
    &\quad + 2 \sum_{\ell=0}^{\lfloor \frac{T-t-1}{m}\rfloor-1}
    \alpha^\ell
    \sum_{i=0}^{m-1}
    \Big(
        \tilde{\varepsilon}_{\tau_{k(t)}+1+\ell m+i}\ D
        + \tilde{\delta}_{\tau_{k(t)}+1+\ell m+i}
    \Big),
\end{align*}
and
\begin{align*}
    \mathcal{E}^{\mathrm{state}}_{t}
    &:= \max_{s\in\mathcal{S}} \mathrm{sp}\big(r_t(s,\cdot)\big)
    + \max_{s\in\mathcal{S}}\max_{a,a'\in\mathcal{A}}
    \|P_t(\cdot\mid s,a)-P_t(\cdot\mid s,a')\|_{\mathrm{tv}}\ D.
\end{align*}
Moreover,
\[
    \bar{\varepsilon}_{\tau_{k(t)},t}
    := \max_{s,a}
    \|P_t(\cdot\mid s,a)
    - P_{\tau_{k(t)}}(\cdot\mid s,a)\|_{\mathrm{tv}},
\]
and
\[
    \bar{\delta}_{\tau_{k(t)},t}
    := \max_s \mathrm{sp}
    \big(r_t(s,\cdot)-r_{\tau_{k(t)}}(s,\cdot)\big).
\]
For the time-shifted auxiliary MDP in Appendix~\ref{app_3},
$\tilde{\varepsilon}_j :=
\max_{s,a}\|P_j(\cdot\mid s,a)-P_{t+j-\tau_{k(t)}}(\cdot\mid s,a)\|_{\mathrm{tv}}$
and
$\tilde{\delta}_j :=
\max_{s,a}|r_j(s,a)-r_{t+j-\tau_{k(t)}}(s,a)|$
for $j=\tau_{k(t)}+1,\dots,\tau_{k(t)}+T-t-1$.
Empty sums are interpreted as zero; in particular,
$\mathcal{E}^{\mathrm{drift}}_{\tau_{k(t)},T-1}=0$.
The bound uses the value-span bounds and block-contraction condition
specified in the update-time analysis.
\end{lemma}

\medskip
The proof of this lemma is provided in Appendix~\ref{app_3}.
By combining the results of these two lemmas and summing over the whole time horizon, we obtain
\begin{align*}
    \mathcal{DR}(T)
    &\leq
    \max_{s_0\in\mathcal{S}}\sum_{t \in \mathcal{T}_{\mathrm{upd}}} \mathcal{R}_t + \sum_{t \in \mathcal{T}_{\mathrm{skip}}}
    \Big(
        \mathcal{R}_{\tau_{k(t)}}
        + \alpha^{\lfloor \frac{T-t-1}{m}\rfloor}\ D
        + \bar{e}_{\tau_{k(t)},t}
        + \mathcal{E}^{\mathrm{drift}}_{\tau_{k(t)},t}
        + \mathcal{E}^{\mathrm{state}}_t
    \Big),
\end{align*}
where
\begin{align*}
    \mathcal{R}_t
    &:= \alpha^{\lfloor \frac{H_{k(t)}-1}{m} \rfloor}\ D
    + 2\hat{\varepsilon}_{t,0}\ D
    + \mathcal{E}^{\mathrm{est}}_t,\\
    \mathcal{E}^{\mathrm{est}}_t
    &:= 2 \sum_{\ell=0}^{\lfloor \frac{H_{k(t)}-1}{m} \rfloor-1}
    \alpha^\ell
    \sum_{i=1}^m \hat{\varepsilon}_{t,\ell m+i}\ D + 2 \alpha^{\lfloor \frac{H_{k(t)}-1}{m} \rfloor}
    \sum_{i=1}^{H_{k(t)}-\lfloor \frac{H_{k(t)}-1}{m} \rfloor m-1}
    \hat{\varepsilon}_{t,\lfloor \frac{H_{k(t)}-1}{m} \rfloor m+i}\ D.
\end{align*}
These can be written in order terms as
\[
    \mathcal{E}^{\mathrm{est}}_t
    = \mathcal{O}\!\left(
    \frac{1-\alpha^{\lceil\frac{H_{k(t)}-1}{m}\rceil}}{1-\alpha}
    \max_{0\leq i<H_{k(t)}}\hat{\varepsilon}_{t,i}\cdot mD
    \right),
\]
and, for $t<T-1$,
\[
    \mathcal{E}^{\mathrm{drift}}_{\tau_{k(t)},t}
    = \mathcal{O}\!\left(
    \frac{1-\alpha^{\lceil\frac{T-t-1}{m}\rceil}}{1-\alpha}
    \max_{\tau_{k(t)}+1\leq j\leq\tau_{k(t)}+T-t-1}
    \big(\tilde{\varepsilon}_jD+\tilde{\delta}_j\big)\cdot m
    \right).
\]
The additional term $\mathcal{E}^{\mathrm{state}}_t$ captures the loss due to evaluating the last updated policy at the propagated-MAP state estimate rather than at the true state.
Thus the skip-time contribution contains both the time-variation drift term $\mathcal{E}^{\mathrm{drift}}_{\tau_{k(t)},t}$ and the state-mismatch term $\mathcal{E}^{\mathrm{state}}_t$, yielding the combined bound above.
This concludes the proof.

\subsection{Analysis of Dynamic Regret at Update Times (Proof of Lemma~\ref{lem:Q_value_diff_obs})}
\label{app_2}
We now bound the regret incurred at update times.

We start the proof by introducing an auxiliary finite-horizon TVMDP at any 
update time $t \in \mathcal{T}_{\mathrm{upd}}$, $\hat{\mathcal{M}}_t$ as:
\begin{align*}
    \hat{\mathcal{M}}_t
    = \Big(\mathcal{S}, \mathcal{A}, \hat{P}_{t-1},
    \{r_{t+h}\}_{h=0}^{H_{k(t)}-1}, H_{k(t)}\Big)
\end{align*}
with transition kernel $\hat{P}_{t-1}$, which is the maximum likelihood estimate 
computed at update time $t$, fixed for the whole horizon 
$H_{k(t)} = \min\{\bar{H}, T-t\}$, and the same rewards $\{r_{t+h}\}_{h=0}^{H_{k(t)}-1}$ 
as the original TVMDP.

We denote its value function by $\hat{W}^{\star}_{t+h,k(t)}$ as introduced in 
the notation and definitions above, which satisfies
\begin{align*}
    \hat{W}^{\star}_{t+h,k(t)}(s)
    = \max_{a \in \mathcal{A}} \Big\{
    r_{t+h}(s,a)
    + \mathbb{E}_{s'\sim\hat{P}_{t-1}(\cdot\mid s,a)}
    \big[\hat{W}^{\star}_{t+h+1,k(t)}(s')\big]
    \Big\},
    \qquad h=0,\cdots,H_{k(t)}-1,
\end{align*}
for any $s \in \mathcal{S}$, with terminal condition $\hat{W}^\star_{t+H_{k(t)},k(t)} \equiv 0$.

In addition, we denote its state--action value function by $\hat{Z}^{\star}_{t,h}$ as
\begin{align*}
    \hat{Z}^{\star}_{t,h}(s,a)
    = r_{t+h}(s,a)
    + \mathbb{E}_{s'\sim\hat{P}_{t-1}(\cdot\mid s,a)}
    \big[\hat{W}^{\star}_{t+h+1,k(t)}(s')\big],
    \qquad h=0,\cdots,H_{k(t)}-1,
\end{align*}
for any $(s,a) \in \mathcal{S}\times\mathcal{A}$, where
$\hat{W}^{\star}_{t+h,k(t)}(s)
= \max_{a\in\mathcal{A}}\hat{Z}^{\star}_{t,h}(s,a)$.

Now we upper bound $\Delta_t(s_0)$. For that, rewrite
$Q_t^\star(s_t,a_t^\star)-Q_t^\star(s_t,a_t^{\mathrm{alg}})$
by adding and subtracting $\hat{Z}_{t,0}^\star(s_t,a_t^{\mathrm{alg}})$
and $\hat{Z}_{t,0}^\star(s_t,a_t^\star)$:
\begin{align*}
    Q_t^\star(s_t,a_t^\star)-Q_t^\star(s_t,a_t^{\mathrm{alg}})
    &= Q_t^\star(s_t,a_t^\star)-\hat{Z}_{t,0}^\star(s_t,a_t^\star)
    + \hat{Z}_{t,0}^\star(s_t,a_t^{\mathrm{alg}})
    - Q_t^\star(s_t,a_t^{\mathrm{alg}})\\
    &\quad + \hat{Z}_{t,0}^\star(s_t,a_t^\star)
    - \hat{Z}_{t,0}^\star(s_t,a_t^{\mathrm{alg}}).
\end{align*}

Since $a_t^{\mathrm{alg}}=\arg\max_a\hat{Z}_{t,0}^\star(s_t,a)$, we have $\hat{Z}_{t,0}^\star(s_t,a_t^\star)
    - \hat{Z}_{t,0}^\star(s_t,a_t^{\mathrm{alg}})\leq0$, 
and therefore
\begin{align*}
    Q_t^\star(s_t,a_t^\star)-Q_t^\star(s_t,a_t^{\mathrm{alg}})
    &\leq \max_a\Big(Q_t^\star(s_t,a)-\hat{Z}_{t,0}^\star(s_t,a)\Big)
    - \min_a\Big(Q_t^\star(s_t,a)-\hat{Z}_{t,0}^\star(s_t,a)\Big).
\end{align*}

Expanding the terms and using that the rewards cancel since both MDPs share 
the same reward functions, we obtain
\begin{align*}
    &Q_t^\star(s_t,a_t^\star)-Q_t^\star(s_t,a_t^{\mathrm{alg}})\\
    &\quad \leq \max_{a\in\mathcal{A}}\bigg(
    \mathbb{E}_{s'\sim P_t(\cdot\mid s_t,a)}[V_{t+1}^\star(s')]
    - \mathbb{E}_{s'\sim\hat{P}_{t-1}(\cdot\mid s_t,a)}
    [\hat{W}_{t+1,k(t)}^\star(s')]
    \bigg)\\
    &\qquad - \min_{a\in\mathcal{A}}\bigg(
    \mathbb{E}_{s'\sim P_t(\cdot\mid s_t,a)}[V_{t+1}^\star(s')]
    - \mathbb{E}_{s'\sim\hat{P}_{t-1}(\cdot\mid s_t,a)}
    [\hat{W}_{t+1,k(t)}^\star(s')]
    \bigg)\\
    &\quad = \mathrm{sp}\bigg(
    \mathbb{E}_{s'\sim P_t(\cdot\mid s_t,\cdot)}[V_{t+1}^\star(s')]
    - \mathbb{E}_{s'\sim\hat{P}_{t-1}(\cdot\mid s_t,\cdot)}
    [\hat{W}_{t+1,k(t)}^\star(s')]
    \bigg).
\end{align*}

Now, adding and subtracting
$\mathbb{E}_{s'\sim\hat{P}_{t-1}(\cdot\mid s_t,a)}[V_{t+1}^\star(s')]$, we obtain
\begin{align*}
    Q_t^\star(s_t,a_t^\star)-Q_t^\star(s_t,a_t^{\mathrm{alg}})
    &\leq \mathrm{sp}\bigg(
    \mathbb{E}_{s'\sim P_t(\cdot\mid s_t,\cdot)}[V_{t+1}^\star(s')]
    - \mathbb{E}_{s'\sim\hat{P}_{t-1}(\cdot\mid s_t,\cdot)}
    [V_{t+1}^\star(s')]
    \bigg)\\
    &\quad + \mathrm{sp}\bigg(
    \mathbb{E}_{s'\sim\hat{P}_{t-1}(\cdot\mid s_t,\cdot)}
    [V_{t+1}^\star(s')]
    - \mathbb{E}_{s'\sim\hat{P}_{t-1}(\cdot\mid s_t,\cdot)}
    [\hat{W}_{t+1,k(t)}^\star(s')]
    \bigg).
\end{align*}

Using the standard bound
\begin{align*}
    \mathrm{sp}\bigg(
    \mathbb{E}_{s'\sim P_t(\cdot\mid s_t,\cdot)}[V_{t+1}^\star(s')]
    &- \mathbb{E}_{s'\sim\hat{P}_{t-1}(\cdot\mid s_t,\cdot)}
    [V_{t+1}^\star(s')]
    \bigg)\\
    &\qquad \leq
    2\max_{a\in\mathcal{A}}
    \|P_t(\cdot\mid s_t,a)-\hat{P}_{t-1}(\cdot\mid s_t,a)\|_{\mathrm{tv}}\,
    \mathrm{sp}(V_{t+1}^\star),
\end{align*}
we get
\begin{equation}\label{eq:helper}
\begin{aligned}
    Q_t^\star(s_t,a_t^\star)&-Q_t^\star(s_t,a_t^{\mathrm{alg}})\\
    &\qquad \leq
    2\max_{s,a}
    \|P_t(\cdot\mid s,a)-\hat{P}_{t-1}(\cdot\mid s,a)\|_{\mathrm{tv}}\,
    \mathrm{sp}(V_{t+1}^\star)+ \mathrm{sp}\Big(
    V_{t+1}^\star-\hat{W}_{t+1,k(t)}^\star
    \Big).
\end{aligned}
\end{equation}

To upper bound $\mathrm{sp}(V_{t+1}^\star-\hat{W}_{t+1,k(t)}^\star)$,
we apply Lemma~\ref{lem:multi_stage_error} under its block-contraction
condition, with $\tilde{P}_\ell=\hat{P}_{t-1}$,
$\tilde{r}_\ell=r_\ell$ for $\ell=t,\dots,t+H_{k(t)}-1$,
and $\tilde{\delta}_\ell=0$ since both MDPs share the same rewards.
This gives
\begin{align*}
    \hat{\varepsilon}_{t,i}
    &= \max_{s,a}
    \|P_{t+i}(\cdot\mid s,a)-\hat{P}_{t-1}(\cdot\mid s,a)\|_{\mathrm{tv}},
\end{align*}
and consequently
\begin{align*}
    \mathrm{sp}\big(V_{t+1}^\star&-\hat{W}_{t+1,k(t)}^\star\big)\\
    &\quad \leq
    \alpha^{\lfloor\frac{H_{k(t)}-1}{m}\rfloor}
    \ \mathrm{sp}\big(V_{t+H_{k(t)}}^\star\big)\\
    &\quad + 2\alpha^{\lfloor\frac{H_{k(t)}-1}{m}\rfloor}
    \sum_{i=1}^{H_{k(t)}-\lfloor\frac{H_{k(t)}-1}{m}\rfloor m-1}
    \hat{\varepsilon}_{t,\lfloor\frac{H_{k(t)}-1}{m}\rfloor m+i}
    \ \mathrm{sp}\bigg(
    \hat{W}_{t+\lfloor\frac{H_{k(t)}-1}{m}\rfloor m+i+1,k(t)}^\star
    \bigg)\\
    &\quad + 2\sum_{\ell=0}^{\lfloor\frac{H_{k(t)}-1}{m}\rfloor-1}
    \alpha^\ell
    \sum_{i=1}^m
    \hat{\varepsilon}_{t,\ell m+i}
    \ \mathrm{sp}\bigg(\hat{W}_{t+\ell m+i+1,k(t)}^\star\bigg),
\end{align*}
where we used that $\hat{W}_{t+H_{k(t)},k(t)}^\star\equiv0$.
When $H_{k(t)}=1$, the same bound follows directly from the terminal
condition, with both sums interpreted as zero.

Going back to~\eqref{eq:helper}, we conclude that
\begin{equation}\label{eq:helper_new}
\begin{aligned}
    &Q_t^\star(s_t,a_t^\star)-Q_t^\star(s_t,a_t^{\mathrm{alg}})\\
    &\quad \leq
    2\hat{\varepsilon}_{t,0}\ \mathrm{sp}(V_{t+1}^\star)
    + \alpha^{\lfloor\frac{H_{k(t)}-1}{m}\rfloor}
    \ \mathrm{sp}(V_{t+H_{k(t)}}^\star)\\
    &\qquad + 2\alpha^{\lfloor\frac{H_{k(t)}-1}{m}\rfloor}
    \sum_{i=1}^{H_{k(t)}-\lfloor\frac{H_{k(t)}-1}{m}\rfloor m-1}
    \hat{\varepsilon}_{t,\lfloor\frac{H_{k(t)}-1}{m}\rfloor m+i}
    \ \mathrm{sp}\bigg(
    \hat{W}_{t+\lfloor\frac{H_{k(t)}-1}{m}\rfloor m+i+1,k(t)}^\star
    \bigg)\\
    &\qquad + 2\sum_{\ell=0}^{\lfloor\frac{H_{k(t)}-1}{m}\rfloor-1}
    \alpha^\ell
    \sum_{i=1}^m
    \hat{\varepsilon}_{t,\ell m+i}
    \ \mathrm{sp}\bigg(\hat{W}_{t+\ell m+i+1,k(t)}^\star\bigg).
\end{aligned}
\end{equation}

Now since this bound holds for any $s_t\in\mathcal{S}$, we conclude that:
\begin{align*}
    \Delta_t(s_0)
    &\leq
    2\hat{\varepsilon}_{t,0}\ \mathrm{sp}(V_{t+1}^\star)
    + \alpha^{\lfloor\frac{H_{k(t)}-1}{m}\rfloor}
    \ \mathrm{sp}(V_{t+H_{k(t)}}^\star)\\
    &\quad + 2\alpha^{\lfloor\frac{H_{k(t)}-1}{m}\rfloor}
    \sum_{i=1}^{H_{k(t)}-\lfloor\frac{H_{k(t)}-1}{m}\rfloor m-1}
    \hat{\varepsilon}_{t,\lfloor\frac{H_{k(t)}-1}{m}\rfloor m+i}
    \ \mathrm{sp}\big(
    \hat{W}_{t+\lfloor\frac{H_{k(t)}-1}{m}\rfloor m+i+1,k(t)}^\star
    \big)\\
    &\quad + 2\sum_{\ell=0}^{\lfloor\frac{H_{k(t)}-1}{m}\rfloor-1}
    \alpha^\ell
    \sum_{i=1}^m
    \hat{\varepsilon}_{t,\ell m+i}
    \ \mathrm{sp}\big(\hat{W}_{t+\ell m+i+1,k(t)}^\star\big).
\end{align*}
Using the value-span bounds
$\mathrm{sp}(V_u^\star)\leq D$ and
$\mathrm{sp}(\hat{W}_{u,k}^\star)\leq D$
specified in Lemma~\ref{lem:Q_value_diff_obs} gives the stated bound,
which concludes the proof.

\subsection{Analysis of Dynamic Regret During Skip Intervals (Proof of Lemma~\ref{lem:Q_value_diff_skip})}
\label{app_3}
We now analyze the regret accumulated during skip intervals.

Suppose $t \in \mathcal{T}_{\mathrm{skip}}$. At time $t$, the skip-update 
algorithm relies on the most recent update at time $\tau_{k(t)}$. In 
particular, the agent uses the last updated policy evaluated at the 
propagated-MAP state estimate $\hat{s}_t$ and executes $a_t^{\mathrm{alg}} = \pi^{\mathrm{alg}}_{\tau_{k(t)}}(\hat{s}_t).$ Now let $L_t(s_t)$ be defined as
\begin{align*}
    L_t(s_t)
    = Q_t^\star\big(s_t,\pi_t^\star(s_t)\big)
    - Q_t^\star\big(s_t,\pi^{\mathrm{alg}}_{\tau_{k(t)}}(\hat{s}_t)\big).
\end{align*}

We rewrite $L_t(s_t)$ by adding and subtracting the terms
$Q^\star_{\tau_{k(t)}}\big(s_t,\pi^\star_{\tau_{k(t)}}(s_t)\big)$,
$Q^\star_{\tau_{k(t)}}\big(s_t,\pi^{\mathrm{alg}}_{\tau_{k(t)}}(s_t)\big)$,
and $Q_t^\star\big(s_t,\pi^{\mathrm{alg}}_{\tau_{k(t)}}(s_t)\big)$ as follows:
\begin{align*}
    &L_t(s_t)\\
    &= Q_t^\star\big(s_t,\pi_t^\star(s_t)\big)
    - Q^\star_{\tau_{k(t)}}\big(s_t,\pi^\star_{\tau_{k(t)}}(s_t)\big)\\
    &\quad + Q^\star_{\tau_{k(t)}}\big(s_t,\pi^\star_{\tau_{k(t)}}(s_t)\big)
    - Q^\star_{\tau_{k(t)}}\big(s_t,\pi^{\mathrm{alg}}_{\tau_{k(t)}}(s_t)\big)\\
    &\quad + Q^\star_{\tau_{k(t)}}\big(s_t,\pi^{\mathrm{alg}}_{\tau_{k(t)}}(s_t)\big)
    - Q_t^\star\big(s_t,\pi^{\mathrm{alg}}_{\tau_{k(t)}}(s_t)\big)\\
    &\quad + Q_t^\star\big(s_t,\pi^{\mathrm{alg}}_{\tau_{k(t)}}(s_t)\big)
    - Q_t^\star\big(s_t,\pi^{\mathrm{alg}}_{\tau_{k(t)}}(\hat{s}_t)\big)\\
    &= Q^\star_{\tau_{k(t)}}\big(s_t,\pi^\star_{\tau_{k(t)}}(s_t)\big)
    - Q^\star_{\tau_{k(t)}}\big(s_t,\pi^{\mathrm{alg}}_{\tau_{k(t)}}(s_t)\big)\\
    &\quad + \Big(
    Q_t^\star\big(s_t,\pi_t^\star(s_t)\big)
    - Q^\star_{\tau_{k(t)}}\big(s_t,\pi^\star_{\tau_{k(t)}}(s_t)\big)
    \Big)
    - \Big(
    Q_t^\star\big(s_t,\pi^{\mathrm{alg}}_{\tau_{k(t)}}(s_t)\big)
    - Q^\star_{\tau_{k(t)}}\big(s_t,\pi^{\mathrm{alg}}_{\tau_{k(t)}}(s_t)\big)
    \Big)\\
    &\quad + Q_t^\star\big(s_t,\pi^{\mathrm{alg}}_{\tau_{k(t)}}(s_t)\big)
    - Q_t^\star\big(s_t,\pi^{\mathrm{alg}}_{\tau_{k(t)}}(\hat{s}_t)\big)\\
    &\leq Q^\star_{\tau_{k(t)}}\big(s_t,\pi^\star_{\tau_{k(t)}}(s_t)\big)
    - Q^\star_{\tau_{k(t)}}\big(s_t,\pi^{\mathrm{alg}}_{\tau_{k(t)}}(s_t)\big)\\
    &\quad + \Big(
    Q_t^\star\big(s_t,\pi_t^\star(s_t)\big)
    - Q^\star_{\tau_{k(t)}}\big(s_t,\pi_t^\star(s_t)\big)
    \Big)
    - \Big(
    Q_t^\star\big(s_t,\pi^{\mathrm{alg}}_{\tau_{k(t)}}(s_t)\big)
    - Q^\star_{\tau_{k(t)}}\big(s_t,\pi^{\mathrm{alg}}_{\tau_{k(t)}}(s_t)\big)
    \Big)\\
    &\quad + Q_t^\star\big(s_t,\pi^{\mathrm{alg}}_{\tau_{k(t)}}(s_t)\big)
    - Q_t^\star\big(s_t,\pi^{\mathrm{alg}}_{\tau_{k(t)}}(\hat{s}_t)\big),
\end{align*}
where in the last inequality we used
$\pi^\star_{\tau_{k(t)}}(s_t)=\arg\max_a Q^\star_{\tau_{k(t)}}(s_t,a)$
and hence
\[
    Q^\star_{\tau_{k(t)}}\big(s_t,\pi_t^\star(s_t)\big)
    \leq Q^\star_{\tau_{k(t)}}\big(s_t,\pi^\star_{\tau_{k(t)}}(s_t)\big).
\]

Then we continue as follows:
\begin{align*}
    L_t(s_t)
    &\leq Q^\star_{\tau_{k(t)}}\big(s_t,\pi^\star_{\tau_{k(t)}}(s_t)\big)
    - Q^\star_{\tau_{k(t)}}\big(s_t,\pi^{\mathrm{alg}}_{\tau_{k(t)}}(s_t)\big)\\
    &\quad + \max_a\Big(Q_t^\star(s_t,a)-Q^\star_{\tau_{k(t)}}(s_t,a)\Big)
    - \min_a\Big(Q_t^\star(s_t,a)-Q^\star_{\tau_{k(t)}}(s_t,a)\Big)\\
    &\quad + Q_t^\star\big(s_t,\pi^{\mathrm{alg}}_{\tau_{k(t)}}(s_t)\big)
    - Q_t^\star\big(s_t,\pi^{\mathrm{alg}}_{\tau_{k(t)}}(\hat{s}_t)\big)\\
    &= \underbrace{
    Q^\star_{\tau_{k(t)}}\big(s_t,\pi^\star_{\tau_{k(t)}}(s_t)\big)
    - Q^\star_{\tau_{k(t)}}\big(s_t,\pi^{\mathrm{alg}}_{\tau_{k(t)}}(s_t)\big)
    }_{\text{(I)}}\\
    &\quad + \underbrace{
    \mathrm{sp}\Big(Q_t^\star(s_t,\cdot)-Q^\star_{\tau_{k(t)}}(s_t,\cdot)\Big)
    }_{\text{(II)}}\\
    &\quad + \underbrace{
    Q_t^\star\big(s_t,\pi^{\mathrm{alg}}_{\tau_{k(t)}}(s_t)\big)
    - Q_t^\star\big(s_t,\pi^{\mathrm{alg}}_{\tau_{k(t)}}(\hat{s}_t)\big)
    }_{\text{(III)}}.
\end{align*}

\medskip
\noindent
\textbf{Bounding term (I).} Note that for any $s\in\mathcal{S}$, the difference
\[
    Q^\star_{\tau_{k(t)}}\big(s,\pi^\star_{\tau_{k(t)}}(s)\big)
    -
    Q^\star_{\tau_{k(t)}}\big(s,\pi^{\mathrm{alg}}_{\tau_{k(t)}}(s)\big)
    \leq \mathcal{R}_{\tau_{k(t)}}
\]
is bounded by the same argument as in the update-time analysis in
Appendix~\ref{app_2}, yielding the bound in~\eqref{eq:helper_new}
evaluated at $\tau_{k(t)}$. Here
$\mathcal{R}_{\tau_{k(t)}} :=
\alpha^{\lfloor(H_{k(t)}-1)/m\rfloor}D
+2\hat{\varepsilon}_{\tau_{k(t)},0}D
+\mathcal{E}^{\mathrm{est}}_{\tau_{k(t)}}$
is the uniform update-time bound, using the value-span bounds from that
analysis. Since this bound holds for every $s$, it applies at $s=s_t$.

\medskip
\noindent
\textbf{Bounding term (II).} We write
\begin{align*}
    \mathrm{sp}\Big(Q_t^\star(s,\cdot)-Q^\star_{\tau_{k(t)}}(s,\cdot)\Big)
    &\leq \max_s\mathrm{sp}\big(r_t(s,\cdot)-r_{\tau_{k(t)}}(s,\cdot)\big)\\
    &\quad + 2\max_{s,a}
    \|P_t(\cdot\mid s,a)-P_{\tau_{k(t)}}(\cdot\mid s,a)\|_{\mathrm{tv}}\,
    \mathrm{sp}(V_{t+1}^\star)\\
    &\quad + \mathrm{sp}\big(V_{t+1}^\star-V_{\tau_{k(t)}+1}^\star\big),
\end{align*}
where we used the total-variation bound from Appendix~\ref{app_2}
and expanded the $Q$-functions using their Bellman definitions. Defining
\begin{align*}
    \bar{\varepsilon}_{\tau_{k(t)},t}
    &:= \max_{s,a}
    \|P_t(\cdot\mid s,a)-P_{\tau_{k(t)}}(\cdot\mid s,a)\|_{\mathrm{tv}},\\
    \bar{\delta}_{\tau_{k(t)},t}
    &:= \max_s\mathrm{sp}\big(r_t(s,\cdot)-r_{\tau_{k(t)}}(s,\cdot)\big),
\end{align*}
we obtain
\begin{align*}
    \mathrm{sp}\Big(Q_t^\star(s,\cdot)-Q^\star_{\tau_{k(t)}}(s,\cdot)\Big)
    \leq \bar{\delta}_{\tau_{k(t)},t}
    +2\bar{\varepsilon}_{\tau_{k(t)},t}\ D
    +\mathrm{sp}\big(V_{t+1}^\star-V_{\tau_{k(t)}+1}^\star\big).
\end{align*}

To bound $\mathrm{sp}(V_{t+1}^\star-V_{\tau_{k(t)}+1}^\star)$,
we introduce the auxiliary TVMDP
$\tilde{\mathcal{M}}=\Big(\mathcal{S},\mathcal{A},T,
\{\tilde{P}_\ell\}_{\ell=0}^{T-1},\{\tilde{r}_\ell\}_{\ell=0}^{T-1}\Big)$
with
\begin{align*}
    \tilde{r}_\ell(s,a)
    &=
    \begin{cases}
        r_\ell(s,a)
        & \ell=0,\dots,\tau_{k(t)}-1,\\[1mm]
        r_{t+j}(s,a)
        & \ell=\tau_{k(t)}+j,\ j=0,\dots,T-t-1,\\[1mm]
        0
        & \ell=T-t+\tau_{k(t)},\dots,T-1,
    \end{cases}\\[4mm]
    \tilde{P}_\ell(\cdot\mid s,a)
    &=
    \begin{cases}
        P_\ell(\cdot\mid s,a)
        & \ell=0,\dots,\tau_{k(t)}-1,\\[1mm]
        P_{t+j}(\cdot\mid s,a)
        & \ell=\tau_{k(t)}+j,\ j=0,\dots,T-t-1,\\[1mm]
        P_{T-1}(\cdot\mid s,a)
        & \ell=T-t+\tau_{k(t)},\dots,T-1.
    \end{cases}
\end{align*}

Let $\{\tilde{V}_\ell^\star\}_{\ell=0}^T$ denote the optimal value functions
of $\tilde{\mathcal{M}}$, satisfying
\begin{align*}
    \tilde{V}_T^\star(s)&=0,\\
    \tilde{V}_\ell^\star(s)&=\max_{a\in\mathcal{A}}
    \left\{
    \tilde{r}_\ell(s,a)
    +\mathbb{E}_{s'\sim\tilde{P}_\ell(\cdot\mid s,a)}
    [\tilde{V}_{\ell+1}^\star(s')]
    \right\},
    \quad \ell=0,\dots,T-1,
\end{align*}
for all $s\in\mathcal{S}$.

\begin{lemma}[Time-shifted auxiliary MDP]\label{lem:time_shift}
The optimal value functions of $\tilde{\mathcal{M}}$,
$\{\tilde{V}_\ell^\star\}_{\ell=0}^T$ satisfy
\begin{align*}
    &\tilde{V}_\ell^\star(s)
    =V_{t-\tau_{k(t)}+\ell}^\star(s),
    \qquad \ell=\tau_{k(t)},\dots,T-t+\tau_{k(t)},\\
    &\tilde{V}_\ell^\star(s)=0,
    \qquad\qquad\quad
    \ell=T-t+\tau_{k(t)},\dots,T,
\end{align*}
for all $s\in\mathcal{S}$.
\end{lemma}
\begin{proof}
Since $\tilde{V}_T^\star(s)=0$ for all $s$, backward induction yields
\begin{align*}
    \tilde{V}_\ell^\star(s)
    =\max_a\Big\{
    0+\mathbb{E}_{s'\sim\tilde{P}_\ell(\cdot\mid s,a)}[0]
    \Big\}=0,
    \qquad \forall s,\ \ell=T-t+\tau_{k(t)},\dots,T-1.
\end{align*}
Thus,
\begin{align*}
    \tilde{V}_{T-t+\tau_{k(t)}}^\star(s)=0=V_T^\star(s),
    \qquad \forall s\in\mathcal{S}.
\end{align*}
We now prove, by backward induction on $j=0,\dots,T-t$, that
\[
    \tilde{V}_{\tau_{k(t)}+j}^\star(s)=V_{t+j}^\star(s),
    \qquad \forall s\in\mathcal{S}.
\]
For $j=T-t$ we established above:
\[
    \tilde{V}_{T-t+\tau_{k(t)}}^\star(s)=0=V_T^\star(s).
\]
Assume the induction hypothesis
\[
    \tilde{V}_{\tau_{k(t)}+j}^\star(s)=V_{t+j}^\star(s),
    \qquad \forall s\in\mathcal{S},
\]
for some $j\in\{1,\dots,T-t\}$.
Using the Bellman optimality equations of $\tilde{\mathcal{M}}$
at time $\tau_{k(t)}+j-1$ and of the original MDP at time $t+j-1$,
and the fact that
\[
    \tilde{r}_{\tau_{k(t)}+j-1}=r_{t+j-1},
    \qquad
    \tilde{P}_{\tau_{k(t)}+j-1}=P_{t+j-1},
\]
together with the induction hypothesis, we obtain
\begin{align*}
    \tilde{V}_{\tau_{k(t)}+j-1}^\star(s)
    &=
    \max_{a\in\mathcal{A}}
    \left\{
        \tilde{r}_{\tau_{k(t)}+j-1}(s,a)
        +
        \sum_{s'}\tilde{P}_{\tau_{k(t)}+j-1}(s'\mid s,a)
        \tilde{V}_{\tau_{k(t)}+j}^\star(s')
    \right\}\\
    &=
    \max_{a\in\mathcal{A}}
    \left\{
        r_{t+j-1}(s,a)
        +
        \sum_{s'}P_{t+j-1}(s'\mid s,a)V_{t+j}^\star(s')
    \right\}\\
    &=V_{t+j-1}^\star(s).
\end{align*}
Thus the claim holds for $j-1$, completing the backward induction.
This completes the proof.
\end{proof}

\noindent
By Lemma~\ref{lem:time_shift},
$\tilde{V}_{\tau_{k(t)}}^\star=V_t^\star$
and $\tilde{V}_{\tau_{k(t)}+1}^\star=V_{t+1}^\star$, so
\begin{align*}
    \mathrm{sp}\big(V_{t+1}^\star-V_{\tau_{k(t)}+1}^\star\big)
    =
    \mathrm{sp}\big(\tilde{V}_{\tau_{k(t)}+1}^\star
    -V_{\tau_{k(t)}+1}^\star\big).
\end{align*}

For $\ell=\tau_{k(t)}+1,\dots,\tau_{k(t)}+T-t-1$, the discrepancies
from Lemma~\ref{lem:multi_stage_error} are
$\tilde{\varepsilon}_\ell
=\max_{s,a}\|P_\ell(\cdot\mid s,a)
-P_{t+\ell-\tau_{k(t)}}(\cdot\mid s,a)\|_{\mathrm{tv}}$
and
$\tilde{\delta}_\ell
=\max_{s,a}|r_\ell(s,a)-r_{t+\ell-\tau_{k(t)}}(s,a)|$.

Under the block-contraction condition required by
Lemma~\ref{lem:multi_stage_error}, for $t<T-1$ we apply the lemma
at time $\tau_{k(t)}+1$ with $N=T-t-1$.
Using $\tilde{V}_{\tau_{k(t)}+T-t}^\star=V_T^\star\equiv0$,
$\mathrm{sp}(V_\ell^\star)\leq D$, and
$\tilde{V}_{\tau_{k(t)}+j}^\star=V_{t+j}^\star$, we obtain
\begin{align*}
    \mathrm{sp}\big(V_{\tau_{k(t)}+1}^\star
    &-\tilde{V}_{\tau_{k(t)}+1}^\star\big)\\
    &\leq \alpha^{\lfloor\frac{T-t-1}{m}\rfloor}\ D\\
    &\quad +2\alpha^{\lfloor\frac{T-t-1}{m}\rfloor}
    \sum_{i=0}^{T-t-2-\lfloor\frac{T-t-1}{m}\rfloor m}
    \bigg(
        \tilde{\varepsilon}_{\tau_{k(t)}+1+\lfloor\frac{T-t-1}{m}\rfloor m+i}\ D
        +\tilde{\delta}_{\tau_{k(t)}+1+\lfloor\frac{T-t-1}{m}\rfloor m+i}
    \bigg)\\
    &\quad +2\sum_{\ell=0}^{\lfloor\frac{T-t-1}{m}\rfloor-1}
    \alpha^\ell
    \sum_{i=0}^{m-1}
    \bigg(
        \tilde{\varepsilon}_{\tau_{k(t)}+1+\ell m+i}\ D
        +\tilde{\delta}_{\tau_{k(t)}+1+\ell m+i}
    \bigg).
\end{align*}
When $t=T-1$, the same bound follows directly from $V_T^\star\equiv0$
and $\mathrm{sp}(V_{\tau_{k(t)}+1}^\star)\leq D$, with both sums
interpreted as zero.

\medskip
\noindent
\textbf{Bounding term (III).} The difference
\begin{align*}
    Q_t^\star\big(s_t,\pi^{\mathrm{alg}}_{\tau_{k(t)}}(s_t)\big)
    -Q_t^\star\big(s_t,\pi^{\mathrm{alg}}_{\tau_{k(t)}}(\hat{s}_t)\big)
\end{align*}
arises because the policy $\pi^{\mathrm{alg}}_{\tau_{k(t)}}$
is evaluated at $s_t$ versus $\hat{s}_t$, two potentially different states.
Let $a=\pi^{\mathrm{alg}}_{\tau_{k(t)}}(s_t)$ and
$\hat{a}=\pi^{\mathrm{alg}}_{\tau_{k(t)}}(\hat{s}_t)$
denote the two actions. Expanding using the definition of $Q_t^\star$:
\begin{align*}
    &Q_t^\star(s_t,a)-Q_t^\star(s_t,\hat{a})\\
    &\quad =r_t(s_t,a)
    +\mathbb{E}_{s'\sim P_t(\cdot\mid s_t,a)}
    [V_{t+1}^\star(s')]
    -r_t(s_t,\hat{a})
    -\mathbb{E}_{s'\sim P_t(\cdot\mid s_t,\hat{a})}
    [V_{t+1}^\star(s')]\\
    &\quad \leq
    \max_{a\in\mathcal{A}}r_t(s_t,a)
    -\min_{a\in\mathcal{A}}r_t(s_t,a)
    +\mathbb{E}_{s'\sim P_t(\cdot\mid s_t,a)}
    [V_{t+1}^\star(s')]
    -\mathbb{E}_{s'\sim P_t(\cdot\mid s_t,\hat{a})}
    [V_{t+1}^\star(s')]\\
    &\quad =
    \mathrm{sp}\big(r_t(s_t,\cdot)\big)
    +\mathbb{E}_{s'\sim P_t(\cdot\mid s_t,a)}
    [V_{t+1}^\star(s')]
    -\mathbb{E}_{s'\sim P_t(\cdot\mid s_t,\hat{a})}
    [V_{t+1}^\star(s')].
\end{align*}
For the remaining term, define
\[
    \mu(s'):=P_t(s'\mid s_t,a)-P_t(s'\mid s_t,\hat{a}).
\]
Since both transition kernels are probability distributions,
\[
    \sum_{s'\in\mathcal{S}}\mu(s')
    =
    \sum_{s'\in\mathcal{S}}P_t(s'\mid s_t,a)
    -
    \sum_{s'\in\mathcal{S}}P_t(s'\mid s_t,\hat{a})
    =
    1-1
    =
    0.
\]
Hence, for any constant $c\in\mathbb{R}$,
\[
    c\sum_{s'\in\mathcal{S}}\mu(s')=0.
\]
Therefore,
\begin{align*}
    \sum_{s'\in\mathcal{S}}\mu(s')V_{t+1}^\star(s')
    &=
    \sum_{s'\in\mathcal{S}}\mu(s')V_{t+1}^\star(s')
    -
    c\sum_{s'\in\mathcal{S}}\mu(s') =
    \sum_{s'\in\mathcal{S}}\mu(s')
    \big(V_{t+1}^\star(s')-c\big).
\end{align*}
Choosing
\[
    c=\frac{1}{2}
    \big(\max_{s'}V_{t+1}^\star(s')+\min_{s'}V_{t+1}^\star(s')\big),
\]
so that
\[
    |V_{t+1}^\star(s')-c|
    \leq\frac{1}{2}\mathrm{sp}(V_{t+1}^\star)
\]
for all $s'$, we obtain
\begin{align*}
    \sum_{s'\in\mathcal{S}}\mu(s')V_{t+1}^\star(s')
    &\leq
    \sum_{s'\in\mathcal{S}}|\mu(s')|
    \cdot|V_{t+1}^\star(s')-c|\\
    &\leq
    \frac{1}{2}\mathrm{sp}(V_{t+1}^\star)
    \sum_{s'\in\mathcal{S}}|\mu(s')|\\
    &=
    \mathrm{sp}(V_{t+1}^\star)
    \|P_t(\cdot\mid s_t,a)-P_t(\cdot\mid s_t,\hat{a})\|_{\mathrm{tv}}\\
    &\leq D
    \max_{a,a'\in\mathcal{A}}
    \|P_t(\cdot\mid s_t,a)-P_t(\cdot\mid s_t,a')\|_{\mathrm{tv}},
\end{align*}
where we used
\[
    \sum_{s'}|\mu(s')|
    =
    2\|P_t(\cdot\mid s_t,a)-P_t(\cdot\mid s_t,\hat{a})\|_{\mathrm{tv}}
\]
and $\mathrm{sp}(V_{t+1}^\star)\leq D$ in the last two steps.

\medskip
\noindent
\textbf{Combining all bounds.} Since
$L_t(s_t)=Q_t^\star(s_t,a_t^\star)-Q_t^\star(s_t,a_t^{\mathrm{alg}})$
and $\Delta_t(s_0)$ is the expectation of $L_t(s_t)$ over the trajectory
induced by $\pi^{\mathrm{alg}}$, combining terms (I), (II), and (III)
and taking the expectation yields
\begin{align*}
    \Delta_t(s_0)
    &\leq
    \mathcal{R}_{\tau_{k(t)}}
    +\alpha^{\lfloor\frac{T-t-1}{m}\rfloor}\ D
    +\bar{\delta}_{\tau_{k(t)},t}
    +2\bar{\varepsilon}_{\tau_{k(t)},t}\ D\\
    &\quad +2\alpha^{\lfloor\frac{T-t-1}{m}\rfloor}
    \sum_{i=0}^{T-t-2-\lfloor\frac{T-t-1}{m}\rfloor m}
    \bigg(
        \tilde{\varepsilon}_{\tau_{k(t)}+1+\lfloor\frac{T-t-1}{m}\rfloor m+i}\ D
        +\tilde{\delta}_{\tau_{k(t)}+1+\lfloor\frac{T-t-1}{m}\rfloor m+i}
    \bigg)\\
    &\quad +2\sum_{\ell=0}^{\lfloor\frac{T-t-1}{m}\rfloor-1}
    \alpha^\ell
    \sum_{i=0}^{m-1}
    \bigg(
        \tilde{\varepsilon}_{\tau_{k(t)}+1+\ell m+i}\ D
        +\tilde{\delta}_{\tau_{k(t)}+1+\ell m+i}
    \bigg)\\
    &\quad +\max_{s\in\mathcal{S}}\mathrm{sp}\big(r_t(s,\cdot)\big)
    +\max_{s\in\mathcal{S}}\max_{a,a'\in\mathcal{A}}
    \|P_t(\cdot\mid s,a)-P_t(\cdot\mid s,a')\|_{\mathrm{tv}}\ D\\
    &=
    \mathcal{R}_{\tau_{k(t)}}
    +\alpha^{\lfloor\frac{T-t-1}{m}\rfloor}\ D
    +\bar{e}_{\tau_{k(t)},t}
    +\mathcal{E}^{\mathrm{drift}}_{\tau_{k(t)},t}
    +\mathcal{E}^{\mathrm{state}}_t,
\end{align*}
using the definitions in Lemma~\ref{lem:Q_value_diff_skip}.
This concludes the proof.

\end{appendix}

\section{Experiment Details}\label{app:expt}

This appendix gives implementation details for the experiments in Section~\ref{sec:experiments}. All methods use the same update budget, the same finite-horizon planner, and the same propagated-MAP skip execution semantics. They differ only in the rule used to choose update times.

\subsection{Simulation Environment}
\label{app:sim_env}

The simulation environment is the Mars-rover navigation task from Example~\ref{ex:mars_rover}. The state space is an $n\times n$ grid,
$\mathcal{S}=\{0,\ldots,n-1\}\times\{0,\ldots,n-1\}$. Unless otherwise stated, the agent starts at $s_0=(0,0)$ and the goal is $g=(n-1,n-1)$. The action space is
$\mathcal{A}=\{\textsc{right},\textsc{left},\textsc{down},\textsc{up}, \textsc{stay}\}$. Let $F(s,a)$ denote the deterministic next grid cell obtained by applying action $a$ at state $s$, with moves clipped at the grid boundary.

The time-varying transition kernel is controlled by a slip parameter $\rho_t\in[0,0.6]$. At time $t$, the transition model is
\begin{equation}
\label{eq:app_slip_kernel}
P_t(s'\mid s,a)
=
(1-\rho_t)\mathbf{1}\{s'=F(s,a)\}
+
\frac{\rho_t}{|\mathcal{A}|}
\sum_{\tilde a\in\mathcal{A}}
\mathbf{1}\{s'=F(s,\tilde a)\}.
\end{equation}
Thus, with probability $1-\rho_t$ the commanded action is executed, while with probability $\rho_t$ the realized motion is induced by a uniformly random action.

The rollout reward is based on the realized next state,
\begin{equation}
\label{eq:app_reward_rollout}
R(s_t,a_t,s_{t+1})
=
-\frac{d_1(s_{t+1},g)}{20},
\qquad
d_1((i,j),(i',j'))=|i-i'|+|j-j'|.
\end{equation}
For dynamic programming, this transition-dependent reward is converted to its expected state-action reward under the model used for planning:
\begin{equation}
\label{eq:app_reward_expected}
r_{\widehat P}(s,a)
=
\sum_{s'\in\mathcal{S}}
\widehat P(s'\mid s,a)R(s,a,s').
\end{equation}

At each update time, the agent solves the finite-horizon planning problem from Section~\ref{sec:skip-update}. The update-time transition model is held fixed over the planning horizon, and the first-step policy map is reused until the next update. Near the end of the horizon, the effective planning horizon is
$H_t=\min\{H,T-t\}$.

\subsection{Drift Profiles and Reported Simulation Settings}
\label{app:drift_profiles}

We use three families of drift profiles for the slip parameter $\rho_t$.

\paragraph{Linear drift.}
Linear drift uses $\rho_t=\mathrm{clip}(\rho_0+ct,0,0.6)$, where $\rho_0$ is the initial slip and $c$ is the per-step drift rate.

\paragraph{Sinusoidal drift.}
Sinusoidal drift uses
\begin{equation}
\label{eq:app_sin_drift}
\rho_t
=
\mathrm{clip}\!\left(
\rho_0+A\sin\!\left(\frac{2\pi t}{p}\right),
0,0.6
\right),
\end{equation}
where $A$ is the amplitude and $p$ is the period.

\paragraph{Piecewise drift.}
Piecewise drift uses low-drift intervals separated by high-drift bursts. The profile is specified by an initial slip $\rho_0$, a quiet slope $c_{\mathrm{quiet}}$, and burst intervals $\mathcal{B}=\{(a_j,b_j,c_j)\}_j$. For horizon $T$, interval $(a_j,b_j,c_j)$ corresponds to integer times $[\lfloor a_jT\rfloor,\lfloor b_jT\rfloor)$ with slope $c_j$; all non-burst intervals use slope $c_{\mathrm{quiet}}$. In the reported piecewise settings, $\rho_0=0.03$ and $c_{\mathrm{quiet}}=0.001$.

For diagnostics and update scoring, we use the one-step transition-change signal
\begin{equation}
\label{eq:app_max_eps}
\epsilon_t
=
\max_{s,a,s'}
\left|
P_{t+1}(s'\mid s,a)-P_t(s'\mid s,a)
\right|.
\end{equation}
This is the maximum entrywise kernel change and is not multiplied by the state-action dimension.

Table~\ref{tab:app_reported_settings} lists the simulation settings reported in Table~\ref{tab:sim_all_baselines}. 

\begin{table}[h]
\centering
\caption{\small Reported simulation settings. The column $n/T/H/B$ gives grid size, horizon, planning horizon, and update budget. Piecewise burst intervals are normalized by the horizon.}
\label{tab:app_reported_settings}
\resizebox{\linewidth}{!}{
\begin{tabular}{lcll}
\toprule
Setting & $n/T/H/B$ & Adaptive score & Drift profile \\
\midrule
Piecewise-Late Burst
& 14/80/8/2
& Gap
& piecewise, $\mathcal{B}=\{(0.36,0.50,0.014),(0.72,0.92,0.022)\}$ \\

Piecewise-Two-Burst
& 12/50/7/4
& Gap
& piecewise, $\mathcal{B}=\{(0.20,0.34,0.018),(0.62,0.82,0.018)\}$ \\

Piecewise-Alternating
& 12/50/7/4
& Drift
& piecewise, $\mathcal{B}=\{(0.14,0.24,0.020),(0.44,0.54,0.020),(0.74,0.86,0.016)\}$ \\

Sinusoidal-Medium
& 12/50/7/3
& Gap
& sinusoidal, $\rho_0=0.10$, $A=0.08$, $p=10$ \\

Piecewise-Narrow-Late
& 14/70/8/2
& Gap
& piecewise, $\mathcal{B}=\{(0.76,0.86,0.020)\}$ \\

Piecewise-Sparse-Late
& 12/60/7/4
& Gap+Drift
& piecewise, $\mathcal{B}=\{(0.30,0.44,0.015),(0.76,0.94,0.020)\}$ \\

Piecewise-Sparse-Late
& 12/50/7/5
& Gap+Drift
& piecewise, $\mathcal{B}=\{(0.30,0.44,0.015),(0.76,0.94,0.020)\}$ \\

Piecewise-Three-Short
& 12/60/7/6
& Gap+Drift
& piecewise, $\mathcal{B}=\{(0.12,0.22,0.016),(0.46,0.56,0.020),(0.76,0.88,0.018)\}$ \\
\bottomrule
\end{tabular}
}
\end{table}

\subsection{Budgeted Execution Semantics}
\label{app:budgeted_execution}

All methods use exact-budget execution. For budget $B$, each method uses exactly $B$ updates over the horizon, including the mandatory first update $\tau_0=1$. The transition at time $0$ is treated as a common initialization step shared by all methods; at $\tau_0=1$, the agent observes $(s_0,a_0,s_1)$, constructs the first update-time model, and computes the first policy used for budgeted execution. Thus, all update times reported in the experiments belong to $\{1,\ldots,T-1\}$.

At an update time $t$, the agent observes the current state, resets its internal state estimate, computes the transition model available at that time, replans, and spends one unit of budget:
\begin{equation}
\label{eq:app_update_semantics}
\hat{s}_t\leftarrow s_t,
\qquad
\widehat P_{t-1} \leftarrow P_{t-1}
\quad\text{in the controlled simulator},
\qquad
\pi_t^{\mathrm{alg}}\leftarrow \textsc{Plan}(\widehat P_{t-1},H_t).
\end{equation}

The assignment $\widehat P_{t-1}\leftarrow P_{t-1}$ reflects the controlled simulation implementation: the update-time model is instantiated from the known transition generator to isolate the effect of update scheduling. Conceptually, $\widehat P_{t-1}$ is the transition estimate available after observing the
transition ending at time $t$.
At a skip time, let $\tau=\max\{u\le t:u\in\mathcal{T}_{\mathrm{upd}}\}$ be the
most recent update time. The agent does not observe $s_t$, does not recompute
$\widehat P_{\tau-1}$, and does not replan. It selects

\begin{equation}
\label{eq:app_skip_action}
a_t=\pi_\tau^{\mathrm{alg}}(\hat{s}_t).
\end{equation}
The true state evolves under the current transition kernel,
\begin{equation}
\label{eq:app_true_evolution}
s_{t+1}\sim P_t(\cdot\mid s_t,a_t),
\end{equation}
while the internal state estimate is propagated using the stale update-time model:
\begin{equation}
\label{eq:app_map_propagation}
\hat{s}_{t+1}
=
\arg\max_{s'\in\mathcal{S}}
\widehat P_{\tau-1}(s'\mid \hat{s}_t,a_t).
\end{equation}
Thus, skipped steps use a stale policy and stale model for action selection and state propagation, while rewards and true dynamics are always computed using the current true state and current true transition kernel.

\subsection{Adaptive Score Instantiations}
\label{app:adaptive_scores}

The adaptive method uses the score $S_t$ from Eq.~\eqref{eq:adaptive-score}. Let $\tau$ be the most recent update time. The theorem-level score components are
\begin{equation}
\label{eq:app_score_components}
    R_t:=\alpha^{\lfloor (T-t-1)/m\rfloor}D,
    \qquad
    I_t:=\bar e_{\tau,t},
    \qquad
    A_t:=\mathcal{E}^{\mathrm{drift}}_{\tau,t},
\end{equation}
where $R_t$ is the residual-horizon term, $I_t$ is the immediate skip-drift term, and $A_t$ is the accumulated skip-drift term. The implemented score is
\begin{equation}
\label{eq:app_score}
    S_t=w_1\phi(R_t)+w_2\phi(I_t)+w_3\phi(A_t),
    \qquad
    \phi(x)=\log(1+x).
\end{equation}

The constants $D$, $\alpha$, and $m$ are TVMDP-level quantities and are not
estimated analytically in the experiments. We treat their effect in the score as
part of the adaptive-score hyperparameterization and select the corresponding
values by empirical validation search before evaluating on held-out rollouts.

In the experiments, the drift profile is known through the bounded-drift model, and the one-step change signal $\epsilon_t$ in Eq.~\eqref{eq:app_max_eps} is used to instantiate the skip-drift terms. Since fixed positive scale factors can be absorbed into the score weights and threshold, the implementation uses the normalized proxies
\begin{equation}
\label{eq:app_score_proxy}
    I_t
    =
    \sum_{\ell=\tau-1}^{t-1}\epsilon_\ell,
    \qquad
    A_t
    =
    m\frac{1-\alpha^{\lfloor (T-t-1)/m\rfloor}}{1-\alpha}\,I_t.
\end{equation}
When rewards are time-invariant, as in the simulation experiments, $\bar\delta_{\tau,t}=0$, so the skip-drift terms are fully determined by transition drift.

The score labels in Table~\ref{tab:sim_all_baselines} are shorthand for which terms in Eq.~\eqref{eq:app_score} are active. They are inherited from the implementation names and do not indicate a separate computation of two new finite-horizon value functions at every time step. Table~\ref{tab:app_adaptive_scores} lists the score instantiations used in the reported experiments.

\begin{table}[h]
\centering
\scriptsize
\caption{\small Adaptive score instantiations. Active weights are set to one and inactive weights to zero. The threshold $\lambda$ is the threshold used by the corresponding score family.}
\label{tab:app_adaptive_scores}
\begin{tabular}{lccccl}
\toprule
Adaptive score & $w_1$ & $w_2$ & $w_3$ & $\lambda$ & Active terms \\
\midrule
Gap
& 1 & 0 & 1 & 0.15
& residual horizon + accumulated skip drift \\

Drift
& 1 & 1 & 0 & 0.25
& residual horizon + immediate skip drift \\

Gap+Drift
& 1 & 1 & 1 & 0.30
& residual horizon + immediate skip drift + accumulated skip drift \\
\bottomrule
\end{tabular}
\end{table}

At time $t$, the adaptive rule triggers an update if $S_t\ge\lambda$ or if the pacing condition in Algorithm~\ref{alg:adaptive_update} requires an update to use the exact budget. The score family reported for each setting in Table~\ref{tab:sim_all_baselines} is selected using validation rollouts and then held fixed on the held-out evaluation rollouts.

\subsection{Baselines}
\label{app:baselines}

All baselines in Tables~\ref{tab:sim_all_baselines} and
\ref{tab:crazyflie_results} use the same exact update budget, update-time planner, and propagated-MAP skip execution semantics as the adaptive method. They differ only in the schedule used to choose $\mathcal{T}_{\mathrm{upd}}$.

\begin{table}[htbp]
\centering
\scriptsize
\small
\renewcommand{\arraystretch}{1.15}
\caption{Baselines used in the main simulation and hardware comparisons.}
\label{tab:app_baselines}
\begin{tabularx}{\linewidth}{
    @{}
    l
    >{\raggedright\arraybackslash}X
    >{\raggedright\arraybackslash}p{0.25\linewidth}
    @{}
}
\toprule
Main-text name & Schedule rule & Hyperparameter selection \\
\midrule
Per.
& Uniform periodic updates at rounded
  $\mathrm{linspace}(1,T-1,B)$ times.
& None \\

Per.-BO
& Periodic schedule with validation-selected phase/offset.
& Offset selected on validation rollouts \\

Rand.
& Includes $t=1$, then samples $B-1$ update times uniformly
  from $\{2,\ldots,T-1\}$.
& Rollout seed determines schedule \\

Front
& Uses update times $\{1,2,\ldots,B\}$.
& None \\

Back
& Uses update times $\{1,T-B+1,\ldots,T-1\}$.
& None \\

Drift
& Updates when accumulated transition change since the last
  update exceeds a threshold, with exact-budget pacing.
& Threshold selected on validation rollouts \\

Self
& Forecasts the next time accumulated future drift will
  exceed a threshold and schedules the update there.
& Threshold selected on validation rollouts \\

Lazy
& Updates when value-weighted model drift under the current
  reused plan exceeds a threshold.
& Threshold selected on validation rollouts \\
\bottomrule
\end{tabularx}
\end{table}

For \textsc{Per.-BO}, when $2\leq B\leq T-1$, let
$q=\lfloor(T-2)/(B-1)\rfloor$.
Candidate schedules have the form
\[
\mathcal{T}_{\mathrm{upd}}(o)
= \{1\}\cup
\{1+o+kq : k\in\mathbb{Z}_{\geq0},\ 1+o+kq\leq T-1\},
\]
for offsets $o\in\{1,\ldots,q\}$. If a candidate has fewer
than $B$ unique updates, it is filled with the latest unused
times before $T$; if it has more than $B$, it is truncated
to the earliest $B$ times. The ordinary uniform periodic
schedule is also included in the candidate set, and the
best schedule is selected on validation rollouts.
For $B=1$, the schedule is $\mathcal{T}_{\mathrm{upd}}=\{1\}$.

For \textsc{Drift}, let $\tau$ be the most recent update time and define
\begin{equation}
\label{eq:app_drift_accum}
    D_t^\epsilon:=\sum_{\ell=\tau-1}^{t-1}\epsilon_\ell.
\end{equation}
The drift-threshold score is
\begin{equation}
\label{eq:app_drift_threshold_score}
    D_t^\epsilon
    +
    c_{\mathrm{pace}}\frac{t-\tau}{L_t},
    \qquad
    L_t=\left\lceil\frac{T-t}{B-|\mathcal{T}_{\mathrm{upd}}|}\right\rceil.
\end{equation}
An update is triggered when this score exceeds a validation-selected threshold, unless exact-budget pacing forces an earlier update. Candidate thresholds are generated from
\begin{equation}
\label{eq:app_base_threshold}
    \theta_{\mathrm{base}}
    =
    \frac{\sum_{t=0}^{T-2}\epsilon_t}{\max\{B-1,1\}},
\end{equation}
using multipliers $\{0.25,0.5,0.75,1.0,1.25,1.5\}$, together with
$\max_t\epsilon_t$ and $4T^{-1}\sum_t\epsilon_t$.

For \textsc{Self}, the next update after $\tau$ is scheduled as the first $t>\tau$ satisfying
\begin{equation}
\label{eq:app_self_trigger}
\sum_{\ell=\tau-1}^{t-1}\epsilon_\ell\ge \theta,
\end{equation}
where $\theta$ is selected on validation rollouts from the same threshold set as \textsc{Drift}. This baseline uses the known drift profile to forecast when the next update should occur.

For \textsc{Lazy}, let $W_{\tau,1}$ be the one-step-ahead value function from the finite-horizon plan computed at the most recent update $\tau$. The value-weighted drift score is
\begin{equation}
\label{eq:app_lazy_score}
    \ell_t
    =
    \max_{s,a}
    \left|
    \sum_{s'}
    \left(
    P_t(s'\mid s,a)-P_{\tau-1}(s'\mid s,a)
    \right)
    W_{\tau,1}(s')
    \right|.
\end{equation}
An update is triggered when $\ell_t\ge\theta$, with $\theta$ selected on validation rollouts. Candidate thresholds are generated from positive observed values of $\ell_t$ using the $20$th, $40$th, $60$th, and $80$th percentiles, together with the positive minimum and maximum. If all observed values are zero, the fallback threshold is $10^{-3}$.

The drift-aware baselines use drift or model-change signals available in the controlled experimental model. This makes them strong schedule baselines for isolating the value of update timing; all comparisons still enforce the same budget and skip execution semantics.

\subsection{Validation and Evaluation Protocol}
\label{app:validation_protocol}

All simulation methods are evaluated from the same fixed start state
$s_0=(0,0)$. The benchmark for empirical dynamic regret is the no-budget oracle, which has full knowledge of the complete transition sequence $\{P_t\}_{t=0}^{T-1}$ and observes the true state at every time step. Let $G_j^\star$ be the oracle return on held-out rollout seed $j$, and let $G_j^\mu$ be the return of budgeted method $\mu$ on the same seed. The reported empirical dynamic regret is
\begin{equation}
\label{eq:app_empirical_regret}
    \widehat{\mathcal{DR}}^\mu(T)
    =
    \frac{1}{N}
    \sum_{j=1}^N
    \left(
    G_j^\star-G_j^\mu
    \right).
\end{equation}
Lower values are better. Table~\ref{tab:sim_all_baselines} reports
$N=1000$ held-out rollout seeds.

Adaptive score families, periodic offsets, and baseline thresholds are selected using validation rollouts disjoint from the held-out evaluation rollouts. After validation, the selected score family, offset, or threshold is fixed and evaluated on the held-out rollouts. Methods without tunable hyperparameters, such as \textsc{Per.}, \textsc{Rand.}, \textsc{Front}, and \textsc{Back}, are evaluated directly on held-out rollouts.

For the update-placement diagnostics in Table~\ref{tab:update_diagnostics}, high-drift intervals are defined from the drift profile. For piecewise profiles, the configured burst intervals are treated as high-drift intervals. For sinusoidal profiles, high-drift intervals are contiguous intervals where $\epsilon_t$ is in the top quartile for that setting. Let $\mathcal{U}:=\mathcal{T}_{\mathrm{upd}}\setminus\{1\}$ denote the set of non-initial update times and let $\mathcal{H}$ denote the set of high-drift intervals. Precision and coverage are
\begin{equation}
\label{eq:app_precision_coverage}
\mathrm{precision}
=
\frac{|\{u\in\mathcal{U}:u\in \cup_{I\in\mathcal{H}} I\}|}{|\mathcal{U}|},
\qquad
\mathrm{coverage}
=
\frac{|\{I\in\mathcal{H}:I\cap\mathcal{U}\neq\emptyset\}|}{|\mathcal{H}|}.
\end{equation}
Mean stale drift is the time average of $\sum_{\ell=\tau(t)-1}^{t-1}\epsilon_\ell$, where $\tau(t)$ is the most recent update before $t$. Max stale drift is the maximum of the same quantity over the
horizon. Diagnostics are averaged over the reported simulation settings and held-out rollouts.

\subsection{Additional Simulation Results}
\label{app:additional_simulations}

\paragraph{Frozen Configuration and Imperfect Drift Bounds.}
To assess whether the reported gains depend on per-setting tuning or exact drift information, Table~\ref{tab:frozen_robustness} evaluates a single globally frozen scheduler across all eight simulation settings. We compare exact drift inputs $\tilde{\varepsilon}_t=\varepsilon_t$, inflated bounds $\tilde{\varepsilon}_t=1.5\varepsilon_t$, and noisy conservative bounds $\tilde{\varepsilon}_t=\varepsilon_t(1+U_t)$, where $\tilde{\varepsilon}_t$ denotes the bound supplied to the scheduler. All comparisons use the same exact update budget, and results are reported as mean $\pm1.96\,\mathrm{SE}$ over $1000$ paired held-out seeds. Freezing the configuration increases mean normalized regret from $0.682$ to $0.730$; inflated and noisy conservative bounds increase it to $0.795$ and $0.798$, respectively, relative to $1.000$ for \textsc{Per.-BO}. The frozen scheduler retains lower mean regret than \textsc{Per.-BO} in every reported setting and drift-input condition, although some uncertainty intervals overlap.

\begin{table}[htbp]
\centering
\scriptsize
\caption{\small Frozen-configuration generalization and robustness to imperfect drift bounds. Entries are empirical mean dynamic regret $\pm1.96\,\mathrm{SE}$ over $1000$ paired held-out seeds. The score labels identify the per-setting tuned configurations; the frozen scheduler uses one configuration across all rows and drift-input conditions.}
\label{tab:frozen_robustness}
\resizebox{\linewidth}{!}{%
\begin{tabular}{lccccc}
\toprule
\shortstack[l]{Setting\\
(tuned score; $n/T/H/B$)}
&
\shortstack{Tuned Adapt\\
$\tilde{\varepsilon}_t=\varepsilon_t$}
&
\shortstack{Frozen Adapt\\
$\tilde{\varepsilon}_t=\varepsilon_t$}
&
\shortstack{Frozen Adapt\\
$\tilde{\varepsilon}_t=1.5\varepsilon_t$}
&
\shortstack{Frozen Adapt\\
$\tilde{\varepsilon}_t=\varepsilon_t(1+U_t)$}
&
\shortstack{Per.-BO\\
budget matched}
\\
\midrule

\shortstack[l]{Piecewise-Late Burst\\
Gap; $14/80/8/2$}
& $1.91\pm0.15$
& $2.06\pm0.15$
& $2.23\pm0.16$
& $2.17\pm0.17$
& $2.78\pm0.17$
\\

\shortstack[l]{Piecewise-Two-Burst\\
Gap; $12/50/7/4$}
& $0.60\pm0.05$
& $0.65\pm0.05$
& $0.70\pm0.06$
& $0.69\pm0.07$
& $0.86\pm0.06$
\\

\shortstack[l]{Piecewise-Alternating\\
Drift; $12/50/7/4$}
& $0.58\pm0.05$
& $0.64\pm0.05$
& $0.69\pm0.05$
& $0.71\pm0.06$
& $0.79\pm0.05$
\\

\shortstack[l]{Sinusoidal-Medium\\
Gap; $12/50/7/3$}
& $0.23\pm0.04$
& $0.28\pm0.04$
& $0.34\pm0.05$
& $0.35\pm0.06$
& $0.61\pm0.05$
\\

\shortstack[l]{Piecewise-Narrow-Late\\
Gap; $14/70/8/2$}
& $0.83\pm0.10$
& $0.86\pm0.10$
& $0.91\pm0.11$
& $0.89\pm0.12$
& $0.92\pm0.11$
\\

\shortstack[l]{Piecewise-Sparse-Late\\
Gap+Drift; $12/60/7/4$}
& $0.51\pm0.05$
& $0.53\pm0.05$
& $0.58\pm0.05$
& $0.57\pm0.06$
& $0.78\pm0.05$
\\

\shortstack[l]{Piecewise-Sparse-Late\\
Gap+Drift; $12/50/7/5$}
& $0.34\pm0.03$
& $0.35\pm0.03$
& $0.38\pm0.03$
& $0.39\pm0.04$
& $0.44\pm0.03$
\\

\shortstack[l]{Piecewise-Three-Short\\
Gap+Drift; $12/60/7/6$}
& $0.56\pm0.04$
& $0.59\pm0.04$
& $0.64\pm0.05$
& $0.66\pm0.06$
& $0.89\pm0.05$
\\

\midrule
Mean normalized regret
& $0.682$
& $0.730$
& $0.795$
& $0.798$
& $1.000$
\\
\bottomrule
\end{tabular}%
}
\end{table}

\paragraph{Comparison with the Optimal Offline Update Schedule.}
Table~\ref{tab:offline_schedule_oracle} compares the same frozen scheduler with the optimal offline exact-budget schedule on small instances. With the initial update fixed at $\tau_0=1$, we enumerate all $\binom{T-2}{B-1}$ candidate schedules: $153$ for $T/B=20/3$ and $816$ for $T/B=20/4$. Results are reported as mean $\pm1.96\,\mathrm{SE}$ over $1000$ paired held-out seeds, including the adaptive scheduler's excess regret relative to the offline optimum. The mean excess regret is $0.010$ for linear drift, $0.001$ for sinusoidal drift, $0.005$ for the late-burst profile, and $0.010$ for the two-burst profile. These comparisons quantify the remaining gap to the best offline schedule: the frozen scheduler has lower mean regret than \textsc{Per.-BO} on the sinusoidal and two-burst instances, but higher mean regret on the linear and late-burst instances.

\begin{table}[htbp]
\centering
\scriptsize
\caption{\small Comparison with the optimal offline exact-budget update schedule. Entries are empirical mean dynamic regret $\pm1.96\,\mathrm{SE}$ over $1000$ paired held-out seeds. Adaptive excess is the regret difference between the frozen adaptive scheduler and the offline optimum.}
\label{tab:offline_schedule_oracle}
\resizebox{\linewidth}{!}{%
\begin{tabular}{lccccc}
\toprule
Profile
&
\shortstack{$T/B$\\
(\# schedules)}
&
\shortstack{Offline\\
$\widehat{\mathcal{DR}}_{\mathrm{off}}^\star$}
&
\shortstack{Frozen Adapt\\
$\widehat{\mathcal{DR}}_{\mathrm{adapt}}$}
&
\shortstack{Per.-BO\\
$\widehat{\mathcal{DR}}_{\mathrm{per}}$}
&
\shortstack{Adaptive excess\\
$\widehat{\mathcal{DR}}_{\mathrm{adapt}}
-\widehat{\mathcal{DR}}_{\mathrm{off}}^\star$}
\\
\midrule

Linear
& $20/3\;(153)$
& $0.027\pm0.005$
& $0.037\pm0.005$
& $0.034\pm0.005$
& $0.010\pm0.001$
\\

Sinusoidal
& $20/3\;(153)$
& $0.016\pm0.007$
& $0.017\pm0.007$
& $0.018\pm0.007$
& $0.001\pm0.001$
\\

Piecewise-Late Burst
& $20/3\;(153)$
& $0.013\pm0.003$
& $0.018\pm0.004$
& $0.016\pm0.004$
& $0.005\pm0.001$
\\

Piecewise-Two-Burst
& $20/4\;(816)$
& $0.013\pm0.004$
& $0.023\pm0.004$
& $0.026\pm0.004$
& $0.010\pm0.001$
\\

\bottomrule
\end{tabular}%
}
\vspace{-0.5em}
\end{table}

\subsection{Hardware Experiment Details}
\label{app:hardware_details}

The hardware experiments use a Crazyflie quadrotor in an indoor navigation task. The workspace is discretized using the same grid abstraction as the simulation experiments, with $n/T/H/B=12/60/7/6$ for all methods. We evaluate three obstacle-density settings: no obstacles, sparse obstacles, and dense obstacles. Each density setting is averaged over $15$ independently generated layouts with the same start--goal structure.

The high-level scheduler and planner run offboard. The Crazyflie executes the commanded motion using its onboard low-level stabilization controller. The high-level policy produces discrete motion commands, which are converted into local position or velocity setpoints. During skip intervals, the controller reuses the most recently computed policy and propagates the internal state estimate using the propagated-MAP rule in Eq.~\eqref{eq:app_map_propagation}.
During update steps, the system obtains a fresh localization/state estimate, updates the transition model used by the planner, and replans.

Following the experimental setup of~\citet{puthumanaillam2024weathering}, we induce stochasticity by applying simulated wind disturbances that perturb the controller during execution. For a given layout and trial, all methods use the same obstacle arrangement and disturbance model, so differences in outcome are attributable to the update schedule.

Table~\ref{tab:crazyflie_results} reports three metrics. Success is the fraction of layouts in which the quadrotor reaches the goal within the horizon. Collision is the average number of obstacle safety-radius violations per layout. Cost is the normalized executed trajectory cost with collision penalties:
\begin{equation}
\label{eq:app_hardware_cost}
C
=
\frac{
L_{\mathrm{exec}}
+
\kappa_{\mathrm{coll}}N_{\mathrm{coll}}
}{
L_{\mathrm{ref}}
},
\end{equation}
where $L_{\mathrm{exec}}$ is the executed path length, $N_{\mathrm{coll}}$ is the
number of collision events, $L_{\mathrm{ref}}$ is the nominal shortest
collision-free path length for the layout, and $\kappa_{\mathrm{coll}}$ is fixed
across all methods and layouts. Higher success is better; lower collision and
cost are better. The same baseline definitions and abbreviations from
Table~\ref{tab:sim_all_baselines} are used for hardware experiments. The
adaptive method uses the \textsc{Gap+Drift} score instantiation from
Table~\ref{tab:app_adaptive_scores}.

\end{document}